\documentclass[a4paper,fleqn]{cas-sc} 
\usepackage{calrsfs}

\usepackage[numbers]{natbib}

\usepackage{balance}
\usepackage{xcolor}
\usepackage{soul}
\usepackage{tikz}
\usepackage{tabularray}
\usepackage{amsmath,amssymb,amsfonts}
\usepackage{algorithmicx}
\usepackage{color}
\usepackage{array}
\usepackage{longtable}
\usepackage{rotating}
\usepackage{url}
\usepackage{hyperref}
\usepackage[algo2e]{algorithm2e}   
\usepackage{algpseudocode}

\SetKwFor{For}{for}{do}{end}
\SetKwIF{If}{ElseIf}{Else}{if}{then}{else if}{else}{end}
\SetKwFor{While}{while}{do}{end}

\SetAlgoInsideSkip{smallskip}   
\SetInd{1em}{0.8em}             

\def\tsc#1{\csdef{#1}{\textsc{\lowercase{#1}}\xspace}}
\tsc{WGM}
\tsc{QE}
\tsc{EP}
\tsc{PMS}
\tsc{BEC}
\tsc{DE}

\usepackage{framed} 
\usepackage{multicol} 
\usepackage{nomencl} 

\usepackage{hyperref}
\hypersetup{
    colorlinks=true,
    linkcolor=blue,
    filecolor=magenta,      
    urlcolor=cyan,
    pdftitle={Overleaf Example},
    pdfpagemode=FullScreen,
    }
\usepackage{url}
\usepackage{fullpage}
\usepackage{times}
\usepackage{fancyhdr,graphicx,amsmath,amssymb}
\usepackage{lineno,hyperref}
\usepackage{subcaption}
\usepackage{graphicx}
\usepackage{wrapfig}
\usepackage{multirow}
\usepackage{amsmath}
\usepackage{amsfonts,amsthm, bm}
\usepackage{amssymb}
\usepackage{calrsfs}
\usepackage{rotating}
\usepackage{pdflscape}
\usepackage{lscape}
\usepackage{xcolor}
\usepackage{float}
\usepackage{booktabs}
\usepackage{longtable}
\usepackage{multirow}
\usepackage{xltabular}
\usepackage{caption}
\usepackage{algorithm}
\usepackage{algpseudocode}
\usepackage{lineno}  
\usepackage[table]{xcolor}
\definecolor{mygreen}{RGB}{0,100,0}

\usepackage{placeins} 

\usepackage{float}
\usepackage{nccmath}
 \newtheorem{theorem}{Theorem}
 
 \newtheorem{corollary}[theorem]{Corollary}

    \usepackage{ifthen}
    \renewcommand{\nomgroup}[1]{%
    	\ifthenelse{\equal{#1}{C}}{\item[\textbf{List of Abbreviations}]}{%
    		\ifthenelse{\equal{#1}{V}}{\item[\textbf{Indices}]}{
    			\ifthenelse{\equal{#1}{S}}{\item[\textbf{Variables}]}{}}} 
    }

\makenomenclature

\renewcommand*\nompreamble{\begin{multicols}{2}}

\renewcommand*\nompostamble{\end{multicols}}

\begin{document}

\let\WriteBookmarks\relax
\def\floatpagepagefraction{1}
\def\textpagefraction{.001}


\shortauthors{Quamar et~al.}

\title [mode = title]{A Novel MDP Decomposition Framework for Scalable UAV Mission Planning in Complex and Uncertain Environments}



%
\author[1, 3]{Md Muzakkir Quamar}[type=editor,
                        auid=000,bioid=1,
                        orcid=0009-0003-5550-9760]

\cormark[1]

\fnmark[1]

\ead{g201806920@kfupm.edu.sa}


\credit{Conceptualization of this study, Methodology, Software}

\affiliation[1]{organization={Control and Instrumentation Department},
    addressline={King Fahd University of Petroleum and Minerals}, 
    city={Dhahran},
    postcode={31261}, 
    country={KSA}}

\author[1,2]{Ali Nasir}[style=chinese]
\affiliation[2]{organization={Interdisciplinary Research Center for Intelligent
Manufacturing and Robotics},
    addressline={KFUPM}, 
    city={Dhahran},
    postcode={31261}, 
    country={KSA}}

\author[1,3]{Sami ELFerik}[style=chinese]
\affiliation[3]{organization={Interdisciplinary Research Center for Smart Mobility and Logistics},
    addressline={KFUPM}, 
    city={Dhahran},
    postcode={31261}, 
    country={KSA}}


\credit{Data curation, Writing - Original draft preparation}








\begin{abstract}
This paper presents a scalable and fault-tolerant framework for unmanned aerial vehicle (UAV) mission management in complex and uncertain environments. The proposed approach addresses the computational bottleneck inherent in solving large-scale Markov Decision Processes (MDPs) by introducing a two-stage decomposition strategy. In the first stage, a factor-based algorithm partitions the global MDP into smaller, goal-specific sub-MDPs by leveraging domain-specific features such as goal priority, fault states, spatial layout, and energy constraints. In the second stage, a priority-based recombination algorithm solves each sub-MDP independently and integrates the results into a unified global policy using a meta-policy for conflict resolution. Importantly, we present a theoretical analysis showing that, under mild probabilistic independence assumptions, the combined policy is provably equivalent to the optimal global MDP policy. Our work advances artificial intelligence (AI) decision scalability by decomposing large MDPs into tractable subproblems with provable global equivalence. The proposed decomposition framework enhances the scalability of Markov Decision Processes, a cornerstone of sequential decision-making in artificial intelligence, enabling real-time policy updates for complex mission environments. Extensive simulations validate the effectiveness of our method, demonstrating orders-of-magnitude reduction in computation time without sacrificing mission reliability or policy optimality. The proposed framework establishes a practical and robust foundation for scalable decision-making in real-time UAV mission execution.
\end{abstract}


\maketitle

\section{Introduction}
Unmanned Aerial Vehicles (UAVs) are increasingly becoming indispensable in critical civilian and military domains where efficiency, responsiveness, and adaptability are of utmost importance. Their unique capabilities—such as compact size, low operational cost, vertical takeoff and landing, high maneuverability, rapid deployment, and extended endurance—make them an attractive alternative to traditional aerial platforms. As a result, UAVs are now widely employed in diverse applications ranging from logistics, surveillance, and environmental monitoring to reconnaissance and tactical strike operations. However, as UAV missions grow in complexity and scale, particularly in multi-UAV scenarios, effective mission management becomes a central challenge \cite{rodrigues2025multi}. 

A central challenge in UAV mission planning lies in maximizing operational effectiveness while accounting for critical constraints such as limited resources, mission completion requirements, and platform survivability. Achieving optimal decision-making in such contexts requires carefully balancing competing factors—mission objectives, resource utilization, and UAV endurance \cite{nasir2024fault}. For example, in logistics operations, UAVs need to optimize delivery routes to minimize travel time and energy expenditure while adapting to dynamic urban environments with obstacles and airspace restrictions \cite{liu2022two, baraean2023optimal}. In air combat missions, UAVs must simultaneously execute complex tactical maneuvers, avoid detection or interception, and manage onboard resources to sustain prolonged engagement \cite{quamarMDP}. Similarly,  In disaster response scenarios, UAVs are often required to collect real-time situational data under rapidly evolving conditions, demanding robust trade-offs between communication bandwidth, flight endurance, and coverage reliability \cite{khan2022cooperative}. These examples highlight that effective UAV mission planning is not simply about task execution but about maintaining a continuous balance among multiple, often conflicting, objectives under uncertainty.

Coordinating decision-making under uncertainty, ensuring efficient resource allocation, and maintaining robustness in dynamic and adversarial environments call for advanced planning frameworks that can effectively balance computational efficiency with operational reliability \cite{atyabi2018}. Optimizing UAV performance during missions is inherently a multifaceted problem, encompassing not only navigation and obstacle avoidance but also endurance management, fault tolerance, and adaptability to unpredictable events \cite{quamar2025fault}. Over the years, extensive research has focused on developing optimal path-planning and navigation strategies for UAVs \cite{saeed2022,quamarSSD}, with a variety of optimization approaches proposed to minimize energy consumption and enhance mission efficiency \cite{nasir2024secure,saeed2023energy}. Techniques such as geometric and graph-based methods \cite{radhakrishnan}, artificial potential fields \cite{quamarSIED}, game-theoretic formulations \cite{van2022game}, and sensor fusion strategies have been widely studied and applied for UAV navigation and obstacle avoidance \cite{ali2019path}. While these contributions represent significant progress, most existing frameworks tend to address isolated aspects of the problem space. They often emphasize path optimization or collision avoidance without sufficiently integrating the broader requirements of mission success, resource-aware decision-making, and resilience to real-world challenges such as environmental uncertainties, system malfunctions, or hostile threats. This gap underscores the need for unified, decision-theoretic approaches that can holistically capture the interplay between mission objectives, resource constraints, and dynamic operational risks.

Recent studies emphasize that advancing UAV autonomy requires decision-making frameworks capable of adapting to environmental variability while maintaining mission reliability. For instance, Basharat et al. \cite{basharat2022} provide a detailed survey underscoring the importance of algorithms that allow UAVs to respond to dynamic operational conditions without compromising system integrity or mission success. In parallel, the work in \cite{steen2024military} draws attention to the role of resource management, highlighting how effective allocation of limited onboard resources directly impacts UAV endurance, resilience, and long-term reliability. Another complementary line of research explores the use of queuing-theoretic models for mission task management, where tasks are mapped to service queues in order to support efficient scheduling, prioritization, and execution under uncertainty \cite{abir2023software,khabbaz2019modeling}. Together, these contributions reveal that autonomy in UAV systems hinges not only on navigation and control, but also on the seamless integration of adaptability, resource efficiency, and mission-level coordination. A system that can manage this interplay would be capable of handling more sophisticated operations—such as cooperative multi-UAV missions, persistent surveillance, or disaster-response deployments—with minimal human oversight. Such advancements would not only raise mission success rates but also enhance safety, resilience, and sustainability, even in resource-constrained or adversarial settings. In light of these challenges, Markov Decision Processes have gained prominence as a powerful decision-theoretic tool, offering a principled means to model sequential choices under uncertainty and optimize UAV performance across diverse operational scenarios.

MDPs have long been recognized as a rigorous framework for sequential decision-making in uncertain environments, offering a principled way to optimize system behavior under stochastic dynamics \cite{nasir2024optimized}. The strength of MDPs lies in their ability to represent complex systems through a set of discrete states and possible actions, where transitions between states are governed by probabilistic models. At each decision point, the agent selects an action, transitions probabilistically to a subsequent state, and incurs an associated cost. By aggregating these stepwise costs over the planning horizon, MDPs enable the derivation of policies that minimize the expected cumulative cost, thereby ensuring efficient and reliable mission performance under uncertainty—whether the horizon is finite or infinite \cite{hamadouche2020reward}. This capability has led to their widespread application in UAV operations, particularly for tasks such as navigation in uncertain terrains, real-time obstacle avoidance, and adaptive threat evasion \cite{ragi2013uav}. Nevertheless, most existing UAV-focused applications treat these problems in isolation and do not fully address the broader requirement of balancing mission success with resource management and operational risks. In practice, UAVs must contend with simultaneous challenges—including limited battery life, potential hardware faults, communication delays, and exposure to adversarial threats—where trade-offs between competing objectives are unavoidable. Despite the suitability of MDPs for such multi-faceted problems, the literature still lacks integrated frameworks that jointly consider mission goals, resource allocation, and resilience within a unified formulation. This gap highlights the need for innovative approaches that extend conventional MDP methods toward more scalable, resource-aware, and risk-sensitive mission planning solutions

To address this gap, in our earlier work \cite{quamar2025fault} we developed a comprehensive MDP-based framework for UAV mission management that explicitly integrated fault tolerance, energy constraints, and threat evasion into a unified decision-making model . The framework accounted for post-fault UAV capabilities, battery state-of-charge, repair and recharge opportunities, and adaptive navigation strategies, thereby enabling UAVs to sustain mission performance even under sensor failures, actuator faults, or adversarial conditions. A cost-based policy optimization was employed to minimize expected mission costs while ensuring resilience and operational safety in threat-prone environments . Simulation case studies demonstrated that the framework could dynamically reallocate goals, mitigate random threats, and preserve UAV health, confirming its robustness and practical viability . However, a critical limitation observed in this study was scalability: as the number of mission goals or grid locations increased, the state space grew exponentially, leading to prohibitively high computational demands. This “state explosion” effect made solving large-scale mission scenarios computationally expensive and impractical for real-time operations. These challenges necessitate new approaches, such as sub-MDP decomposition or hierarchical MDP models, to partition the problem into smaller, more tractable subspaces while preserving decision quality.

To overcome the computational challenges posed by large-scale MDPs, we build on \cite{quamar2025fault} and propose two decomposition algorithms specifically designed for UAV mission planning. From an AI perspective, the proposed decomposition constitutes a scalable decision-making paradigm that directly addresses the “curse of dimensionality” in sequential decision-making, a key bottleneck in reinforcement learning and stochastic planning. The first algorithm partitions the global MDP into smaller, more manageable sub-MDPs based on decomposition criteria such as mission goals, geographic regions, fault states, or hybrid combinations. By reducing the state space of each subproblem, this approach significantly improves tractability without compromising decision quality. The second algorithm evaluates these sub-MDPs by computing priority scores using local solvers, such as value iteration, and dynamically assigns UAV agents to the sub-MDPs according to their relative importance. The individual results are then combined to construct a joint global policy, ensuring coherent mission-level coordination while maintaining computational efficiency. This framework thus contributes not only to engineering autonomy but also to fundamental AI scalability.

The contributions of this study are multifold. \textbf{First}, it is the only approach to incorporate post-fault UAV capabilities within a decomposed MDP framework, thereby allowing robust policy adaptation under actuator or sensor failures. \textbf{Second}, the framework explicitly integrates the UAV’s state-of-charge to facilitate range-aware task assignment, which not only prevents deep discharge but also promotes energy-efficient operation. \textbf{Third}, repair and recharge options are modeled to extend UAV endurance and operational resilience. \textbf{Fourth}, the framework addresses real-world complexities by incorporating recurring goals and dynamically generated requests, enhancing reliability under uncertain conditions. \textbf{Fifth}, a detailed scalability analysis is performed to highlight the exponential complexity of large MDP models. \textbf{Sixth}, we develop two novel algorithms—one for decomposition of the global MDP into goal-relevant sub-MDPs, and another for recombination using priority-based meta-policy coordination—to solve the global MDP problem in a more computationally and cost-effective manner. \textbf{Seventh}, we present a formal theoretical proof establishing that, under certain independence assumptions, the recombined policy from sub-MDPs is provably equivalent to the global MDP policy. \textbf{Finally}, we extend and improve upon our earlier work~\cite{quamar2025fault}, to design a scalable, resource-aware, and fault-tolerant UAV mission management framework.
A high-level architecture of the proposed model and its setup is shown in Figure~\ref{BD_PF}.


\begin{figure} 
\centering
\includegraphics[width=\textwidth,keepaspectratio]{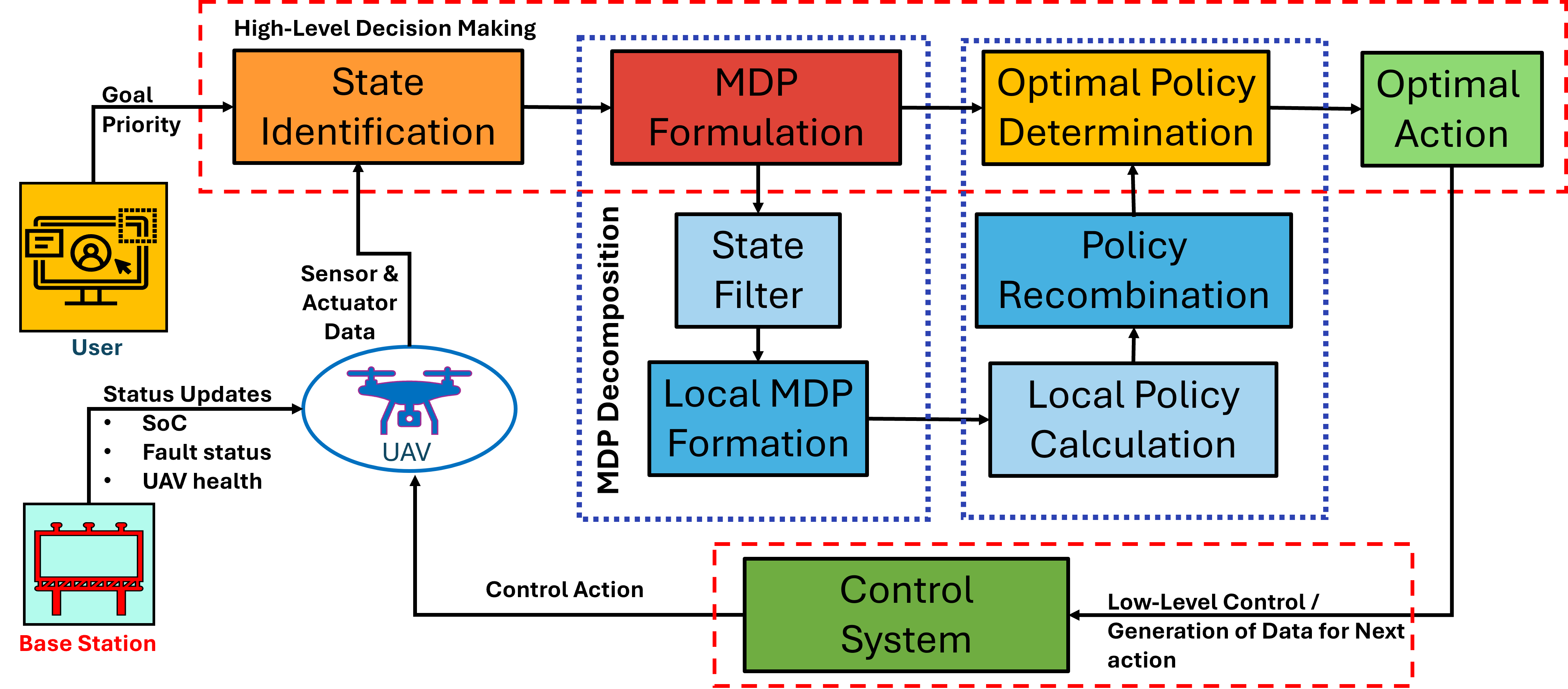}
\caption{Block diagram architecture of the proposed MDP model for single UAV operation.}
\label{BD_PF}  
\end{figure}


Building on these contributions, the framework employs stochastic dynamic programming to optimize dynamic task assignments, providing a systematic approach for decision-making under uncertainty and rapidly changing mission conditions. Since the optimal policies are computed offline, the framework minimizes computational overhead during mission execution, enabling fast responses in real time. Together, these features establish a practical and scalable solution for UAV resource optimization and mission success, ensuring reliability and robustness even in complex, uncertain, and adversarial environments.

\FloatBarrier
\section{Background and Problem Formulation}

This section outlines the necessary preliminaries for the problem under investigation. We begin with the classification of UAV fault modes, followed by the modeling of energy consumption and operational range. Next, we describe the UAV’s agility in the presence of threats, and finally, we summarize the fundamentals of Markov Decision Processes (MDPs), which form the basis of our formulation.

\FloatBarrier
\subsection{Fault Identification and Classification}
We model the UAV as a continuous‐time nonlinear system,
\[
\dot x = f(x) + g(x)\,u,\qquad y = h(x),
\]
where \(x\in\mathbb{R}^n\) is the state, \(u\in\mathbb{R}^m\) the control input, and \(y\in\mathbb{R}^p\) the sensor output.  Under nominal conditions, the vehicle is both controllable and observable.  Actuator or sensor failures—and payload camera malfunctions—can compromise controllability and/or observability.  We adopt the eight discrete fault modes defined in our previous work \cite{quamar2025fault}, which form the basis of our fault‐tolerant MDP model; the detailed classification is discussed in Table- 1 of \cite{quamar2025fault}.

\FloatBarrier
\subsection{Power Consumption and Range}
Battery capacity fundamentally limits UAV endurance (typically 60–90 min).  Let \(\mathrm{SOC}\in[0,1]\) be the normalized state‐of‐charge, \(b_c\) the battery capacity, \(v\) the terminal voltage, and \(P_{\rm UAV}\) the instantaneous power draw.  The flight time satisfies
\[
\tau \le \frac{\mathrm{SOC}\,b_c\,v}{P_{\rm UAV}},
\]
so that at full charge \(\tau_{\max}=\tfrac{b_c\,v}{P_{\rm UAV}}\).  Since \(P_{\rm UAV}\) varies with maneuvering, we approximate the achievable range
\[
\sigma = \int_0^{\tau}\beta(z)\,dz \;-\; k_1\,q \;-\; k_2\,\varrho,
\]
where \(\beta(z)\) is the speed profile, \(q\) the worst‐case wind resistance, \(\varrho\) the payload mass, and \(k_1,k_2\) empirical constants.  A flag \(r_\alpha=1\) indicates that task \(\alpha\) at distance \(d_\alpha\) is within range (\(\sigma\ge d_\alpha\)).

\subsection{Threat Evasion and Agility}
We categorize the environment into three threat levels \(t\in\{0,1,2\}\) (no, low, high) and allow the UAV to switch between
\[
m = 
\begin{cases}
0, &\text{normal navigation (energy‐efficient)};\\
1, &\text{high‐speed evasion (risk‐mitigation).}
\end{cases}
\]
This dynamic mode switching balances energy usage against safety requirements in the presence of adversarial or environmental hazards.


\FloatBarrier
\subsection{Preliminary Discussion on MDP}

Decision-making under uncertainty is a central challenge in UAV mission planning, particularly in environments characterized by dynamic events, faults, and limited resources. MDPs provide a mathematically rigorous framework to address this challenge by modeling sequential decisions in stochastic settings \cite{nasir2018markov}. An MDP formalizes the interaction between the UAV (agent) and its environment, enabling the design of policies that minimize long-term operational costs while accounting for uncertainty in system dynamics.

Formally, an MDP is represented as a tuple:
\[
MDP \; \rightarrow \; \{\mathcal{S}, \mathcal{D}, \mathcal{T}, \mathcal{J}, \gamma\}
\]
where the elements are defined as follows:
\begin{itemize}
    \item $\mathcal{S}$: a finite set of states, $s \in \{s_1, s_2, \ldots, s_n\}$, describing the UAV and environment (e.g., location, battery status, fault condition, threat level).
    \item $\mathcal{D}$: a set of available decisions or actions, $d \in \{d_1, d_2, \ldots, d_m\}$, such as moving to a waypoint, recharging, or entering an agile mode.
    \item $\mathcal{T}(s, d, s')$: the transition probability of moving from state $s$ to $s'$ after applying action $d$, where $\mathcal{T}: \mathcal{S} \times \mathcal{D} \times \mathcal{S} \rightarrow [0,1]$.
    \item $\mathcal{J}(s,d)$: the cost function associated with taking action $d$ in state $s$. In this study, costs are preferred over rewards, capturing penalties due to energy consumption, missed tasks, or risk exposure.
    \item $\gamma \in (0,1)$: the discount factor that balances the importance of immediate versus future costs. A larger $\gamma$ emphasizes long-term mission objectives, while a smaller $\gamma$ prioritizes short-term safety and efficiency.
\end{itemize}

\begin{figure} 
\centering
{\includegraphics[width= 11.5 cm]{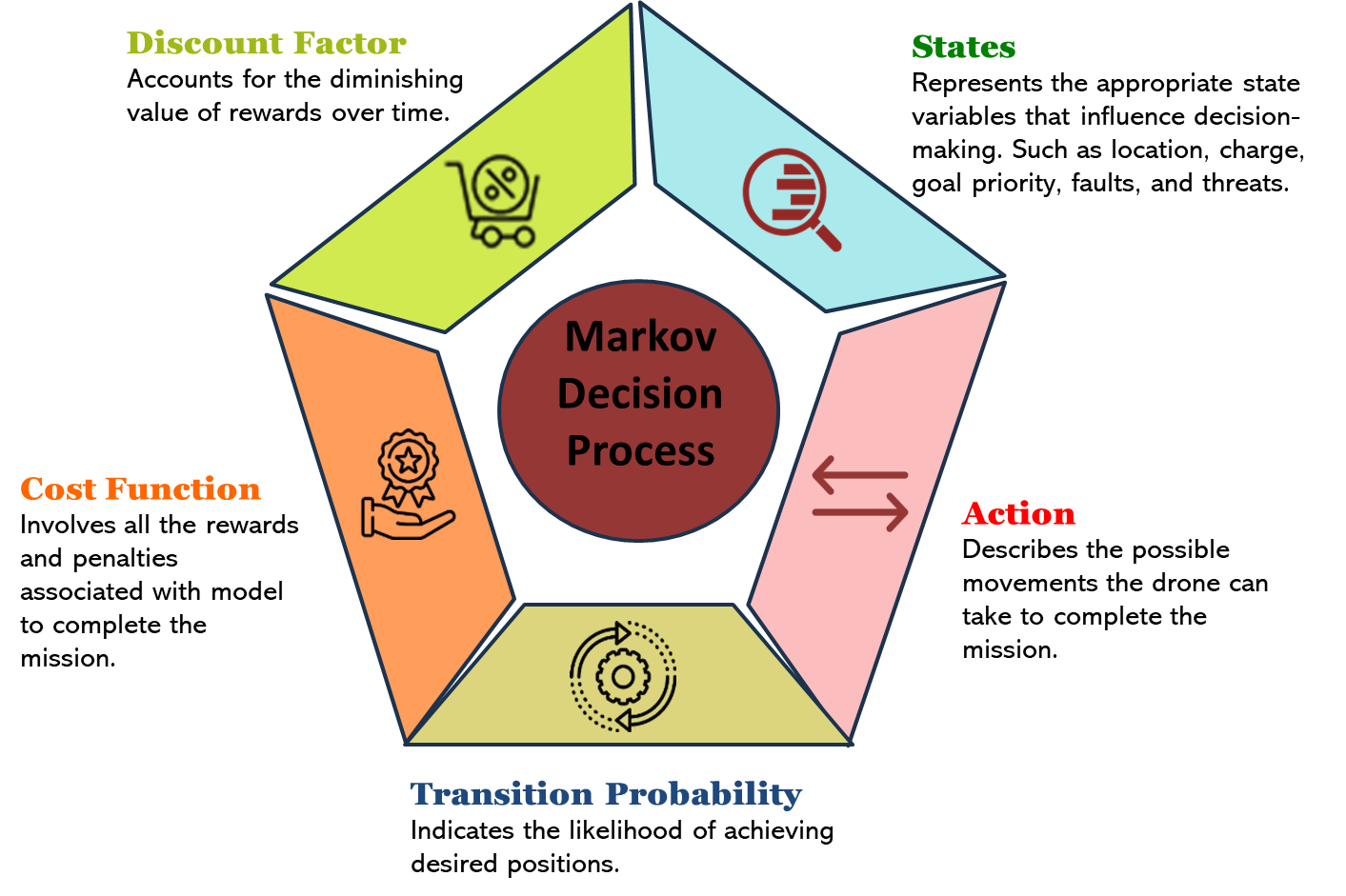}}
\caption{Basic components of an MDP model formulation.}
\label{MDP}  
\end{figure}

The UAV’s objective is to determine an optimal policy $\pi^*(s)$ that maps each state $s \in \mathcal{S}$ to an action $d \in \mathcal{D}$ such that the expected cumulative cost over the mission horizon is minimized. The value of a state under a policy is iteratively updated using Bellman’s optimality equation:  

\begin{equation}
V^{t+1}(s) = \min_{d \in \mathcal{D}} \Bigg\{ \mathcal{J}(s, d) + \gamma \sum_{s' \in \mathcal{S}} \mathcal{T}(s,d,s') \, V^t(s') \Bigg\}
\end{equation}

where $V^t(s)$ is the value of state $s$ at iteration $t$. The iterative process continues until convergence is achieved, i.e.,  

\[
\|V^{t+1} - V^t\|_\infty < \eta,
\]  

with $\eta$ being a small tolerance threshold. Once convergence is reached, the optimal policy is obtained as:  

\begin{equation}
\pi^*(s) = \arg \min_{d \in \mathcal{D}} \Bigg\{ \mathcal{J}(s, d) + \gamma \sum_{s' \in \mathcal{S}} \mathcal{T}(s,d,s') \, V^*(s') \Bigg\}, \quad \forall s \in \mathcal{S}.
\end{equation}

The choice of solution algorithm depends on the problem scale and constraints. In this work, we adopt \textit{value iteration} due to its simplicity and suitability for problems where the complete model is available. Value iteration progressively estimates the optimal state values and subsequently derives the best policy. Although computationally demanding for very large state spaces, it provides a systematic and guaranteed way to obtain an optimal solution.  

Figure~\ref{MDP} illustrates the essential components of an MDP, highlighting the iterative feedback between the UAV’s state, chosen action, probabilistic transitions, and resulting costs. This abstraction not only captures the uncertainty inherent in UAV missions but also provides a foundation for scalable planning techniques, such as the decomposition methods developed in this study.

\FloatBarrier
\subsection{Problem Setup and Formulation}

In this study, we consider a mission scenario where a multiple UAV operates within a finite geographical area. For modeling purposes, the operational region is partitioned into a set of discrete cells or grids. Although these grids are represented uniformly in the schematic, in practice they may differ in size and shape due to geographical or operational constraints. Within this environment, specific targets or goals are distributed, each corresponding to a mission requirement. Some of these regions may also be designated as \textit{threat zones}, the likelihood of which is modeled through a probability distribution. A ground control station, hereafter referred to as the Base Station (BS), oversees the mission, exchanges information with the UAV, and updates mission priorities dynamically. A high-level block diagram of this setup is shown in Figure~\ref{BD_PF}.

The following assumptions define the scope of the problem:  
\begin{itemize}
    \item The BS is responsible for assigning mission goals and updating their priorities during operation in response to situational changes.  
    \item Controllability and observability analyses for the UAV under possible sensor and actuator faults are performed offline. Hence, the UAV’s ability to operate under different fault conditions is known beforehand.  
    \item The UAV is equipped with a fault-tolerant control architecture, including a switching mechanism that enables adaptation when faults in sensors or actuators are detected.  
    \item The state of charge (SoC) of the UAV’s battery is continuously monitored, enabling the system to evaluate energy feasibility for current and future tasks.  
    \item A payload camera is mounted on the UAV, and its operational status is tracked in real time to ensure availability of sensing capabilities during the mission.  
\end{itemize}

\vspace{0.3cm}

\noindent
Under this setup, the UAV’s primary objective is to accomplish the mission goals while managing its limited resources and adapting to environmental uncertainty. Mission examples include package delivery, aerial surveillance, search-and-rescue operations, and environmental monitoring. The UAV must optimize its actions to reduce energy consumption and mission completion time, while simultaneously addressing unexpected challenges such as:  
\begin{itemize}
    \item occurrence of actuator or sensor faults,  
    \item potential loss of communication,  
    \item randomly appearing threats in certain regions, and  
    \item dynamic re-prioritization of mission goals by the BS.  
\end{itemize}

These complexities highlight the need for a mathematical framework that not only supports sequential decision-making but also explicitly accounts for uncertainty and stochasticity in system evolution.

MDP provide a principled stochastic dynamic programming framework for decision-making under uncertainty \cite{puterman2014markov}. By modeling UAV states, decisions, transition probabilities, and costs, an MDP enables the computation of optimal policies that guide UAV actions across different scenarios. The suitability of MDPs for UAV mission management has been widely recognized, as they support probabilistic reasoning, resource-aware optimization, and adaptation to unpredictable events.

The advantages of adopting an MDP formulation in this context are threefold \cite{aldurgam2021mdp,hamid2021reinforcement}:  
\begin{enumerate}
    \item \textbf{Planning Horizons:} MDPs allow derivation of optimal policies for both finite and infinite horizons, with the latter being particularly useful for long-duration or stationary mission models.  
    \item \textbf{Flexible Cost Structures:} The cost function can incorporate diverse operational penalties, such as energy usage, mission delays, or safety risks, depending on the UAV’s current and subsequent states.  
    \item \textbf{Proven Practicality:} MDP-based solutions have been successfully applied across a wide range of real-world stochastic decision-making problems, validating their relevance to UAV applications.  
\end{enumerate}

\noindent
To exploit these capabilities, the UAV mission problem considered here is formulated as an MDP, enabling systematic handling of uncertainty, faults, and dynamic conditions. The next section presents the detailed mathematical modeling of this formulation and the proposed solution framework.



\FloatBarrier
\section{MDP-Based Framework for UAV Missions}

The dynamic task assignment problem considered in this work is illustrated in Figure~\ref{BD_PF}. The UAV is required to complete multiple goals, with their priorities updated in real time by the base station (BS) according to mission needs. To support this, we employ a Markov Decision Process (MDP) framework that captures the UAV’s operational state, mission objectives, and environmental uncertainties.

Within this framework, the UAV evaluates the cost of each feasible action and selects the one that minimizes expected mission cost while respecting priority constraints. Communication with the BS ensures that goal priorities are updated dynamically, allowing the UAV to adjust its policy as conditions change. 

The model further accounts for contingencies: the UAV may return to the BS for repair after a fault, or recharge its battery when energy levels are low. By integrating mission accomplishment with resource management, threat evasion, and fault tolerance, the proposed MDP formulation provides a structured approach for ensuring safe and efficient operations in uncertain and adversarial environments.

\FloatBarrier
\subsection{Definition of State Variables}

A crucial step in formulating an MDP is specifying the state variables that accurately capture the UAV’s operational context. These variables must reflect not only the UAV’s internal health and energy constraints but also its interaction with mission goals and external threats. For this problem, the overall state space $\mathcal{S}$ is given as:

\[
\begin{aligned}
\mathcal{S} &= \{s_1, s_2, \ldots, s_n\}, \quad s_i = \{f_i, R_i, G_i, l_i, c_i, t_i, m_i\}, \quad i \in \{1,2,\ldots,n\} \\
f_i &\in \{1,2,3,4\}, \\
R_i &= \{r_{1,i}, r_{2,i}, \ldots, r_{k,i}\}, \quad r_{j,i} \in \{0,1\}, \quad j \in \{1,\ldots,k\}, \\
G_i &= \{g_{1,i}, g_{2,i}, \ldots, g_{k,i}\}, \quad g_{j,i} \in \{0,1,2\}, \quad j \in \{1,\ldots,k\}, \\
l_i &\in \{1,2,\ldots,q\}, \\
c_i &\in \{0,1,2,\ldots,k\}, \\
t_i &\in \{0,1,2\}, \\
m_i &\in \{0,1\}.
\end{aligned}
\]

Each state $s_i$ is thus defined by eight variables, which together capture the UAV’s internal status and its environment. A brief description of these variables is provided below.

\begin{itemize}
    \item \textbf{Fault status ($f_i$):} Represents the UAV’s health mode, with eight possible categories (Table -1) \cite{quamar2025fault}.
    \item \textbf{Range feasibility ($R_i$):} A set of $k$ binary indicators showing whether the UAV has enough energy to attempt each goal. If $r_{j,i}=1$, the $j$-th goal is reachable given the current SoC; otherwise, $r_{j,i}=0$.  
    \item \textbf{Goal priorities ($G_i$):} A set of $k$ flags specifying the importance of each goal. Values include: $g_{j,i}=0$ (goal achieved or irrelevant), $g_{j,i}=1$ (low priority), and $g_{j,i}=2$ (high priority).  
    \item \textbf{UAV location ($l_i$):} Index of the UAV’s current region in the grid-based environment ($q$ regions in total).  
    \item \textbf{Commitment status ($c_i$):} Indicates whether the UAV is currently committed to any goal ($c_i>0$) or unassigned ($c_i=0$).  
    \item \textbf{Threat level ($t_i$):} Encodes the environmental risk: $t_i=0$ (no threat), $t_i=1$ (low threat), or $t_i=2$ (high threat).  
    \item \textbf{Navigation mode ($m_i$):} Reflects the UAV’s agility setting, where $m_i=0$ denotes normal cruising and $m_i=1$ indicates agile, high-speed evasive maneuvers.  
\end{itemize}

This formulation ensures that all mission-relevant factors—fault tolerance, energy feasibility, mission priorities, spatial information, and threat conditions—are explicitly captured in the state space. By modeling these diverse variables together, the UAV can reason about complex trade-offs, such as whether to commit to a high-priority goal, conserve energy, or evade a threat. The high dimensionality and stochastic behavior of these variables make classical deterministic optimization unsuitable, thereby motivating the use of the Markov Decision Process framework for robust mission management.

\FloatBarrier
\subsection{Decision Variables}

At each decision epoch, the UAV selects an action from a finite set of alternatives that reflect its mission objectives and operational constraints. The available decision set is:

\[
D = \{\text{goal commitment}, \text{service}, \text{recharge}, \text{evasion}\}.
\]

These four actions are defined as follows:
\begin{itemize}
    \item \textbf{Goal commitment:} Allocate resources to pursue a designated mission goal.  
    \item \textbf{Service:} Request maintenance or replacement of faulty onboard components.  
    \item \textbf{Recharge:} Return to the base station to restore battery capacity.  
    \item \textbf{Evasion:} Switch to agile navigation mode to avoid threats.  
\end{itemize}

Actions are event-driven, meaning they are triggered whenever the UAV transitions from one state to another. While the time required for such transitions may vary, temporal dynamics are not explicitly modeled in the MDP. Instead, each decision $d \in D$ alters the system by deterministically updating the commitment variable $c$ and stochastically influencing other factors such as fault status, goal priorities, and threat levels. This causal chain—\textit{state} $s \rightarrow$ \textit{action} $d \rightarrow$ \textit{new state} $s'$—forms the foundation for the policy optimization discussed in later sections.

\FloatBarrier
\subsection{Cost Function}

The MDP model evaluates decisions using a cost-based criterion that captures the main operational challenges faced by a UAV during its mission. The overall cost function depends on both the state and decision variables and integrates five components:

\begin{itemize}
    \item \textbf{Goal-related penalty:} A cost is incurred for unassigned goals that are within feasible range. User-defined constants $\eta_j$ weight the relative importance of each goal.  
    \item \textbf{Distance-based cost:} A location-dependent function $h(d,l_i)$ accounts for the effort needed to reach a goal from the UAV’s current position. This can be modeled, for example, using Euclidean distance.  
    \item \textbf{Fault cost:} The term $\mathfrak{f}(f_i,r_i)$ captures penalties associated with actuator/sensor faults or low battery conditions, reflecting degraded UAV health.  
    \item \textbf{Range violation cost:} If a goal is outside the UAV’s feasible range, an additional penalty is applied. Constants $\delta_j$ quantify the severity of such failures.  
    \item \textbf{Threat-related cost:} The function $p(t_i,m_i)$ penalizes evasive actions in the presence of threats, accounting for the extra energy and risk involved in switching to agile navigation mode.  
\end{itemize}

The consolidated cost function is expressed as:

\begin{equation}
\label{cost}
J(s_i, d) = \sum_{j=1}^k \eta_j g_{j,i} r_{j,i}\big(1 - \mathbb{I}_j(c_i)\big) + h(d,l_i) + \mathfrak{f}(f_i,r_i) + \sum_{j=1}^k \delta_j g_{j,i}(1-r_{j,i}) + p(t_i,m_i).
\end{equation}

Here, each term corresponds to one of the components listed above. The formulation ensures non-negative costs and provides a unified measure for evaluating actions under uncertainty. Detailed examples of the functions $h(d,l_i)$, $\mathfrak{f}(f_i,r_i)$, and $p(t_i,m_i)$ shall be provided in the case study section.

\FloatBarrier
\subsection{Transition Probability}

Among the state variables, two are inherently stochastic: the UAV’s fault condition ($f_i$) and the goal priority levels ($G_i$). Consequently, the transition probability between states is governed by the joint distribution of these two variables. The general form of the transition function is

\begin{equation}
    P(s_i, d, s_j) = \Pr(s_j \mid s_i, d) = \Pr\big(f_j = f, g_{1,j}=g_1, \ldots, g_{k,j}=g_k \mid f_i = \bar{f}, g_{1,i}=\bar{g}_1, \ldots, g_{k,i}=\bar{g}_k, d\big).
\end{equation}

\paragraph{Modeling Assumptions.}  
To simplify the estimation of transition dynamics, the following assumptions are adopted:
\begin{itemize}
    \item Fault occurrences are independent of the evolution of goal priorities.  
    \item Updates in goal priority are not directly influenced by the chosen decision variable.  
\end{itemize}

Under these assumptions, the transition probability can be factorized as

\begin{equation}
  P(s_i, d, s_j) = \Pr\big(g_{1,j}=g_1, \ldots, g_{k,j}=g_k \mid g_{1,i}=\bar{g}_1, \ldots, g_{k,i}=\bar{g}_k, d\big) \; \Pr\big(f_j=f \mid f_i=\bar{f}, d\big).  
\end{equation}

\paragraph{Practical Estimation.}  
In real deployments, transition probabilities are derived from a combination of empirical data and engineering judgment. As more operational data becomes available, these probabilities may be recalibrated, requiring the UAV’s policy to be recomputed accordingly.  

In this work, the fault-transition probabilities are informed by experimental studies of actuator degradation reported in \cite{shahab2025formation}. That study introduced faults at different effectiveness levels (80\%, 60\%, and 40\%) and fit polynomial models to capture the relationship between fault severity and UAV performance. These fitted models are incorporated into our framework to assign probabilities to different failure modes. Although online updating is not performed in this study, the framework allows recalibration whenever new field or laboratory data are available, ensuring that the decision policy remains aligned with the UAV’s operating conditions.

\FloatBarrier
\subsection{Computation of Optimal Values}

State values are obtained via value iteration. We initialize all state values to zero and iteratively update them using the Bellman optimality recursion (cf. Section~2.4) until the updates converge. Convergence is certified by the sup–norm test
\[
\bigl\|V^{t+1}-V^{t}\bigr\|_{\infty} < \eta,
\]
with a small tolerance $\eta$ (typically $<10^{-6}$).

In our setting, the final step differs slightly from the standard policy extraction. Rather than immediately choosing the maximizing action for each state, we compute the \emph{state–action value} for each candidate goal–assignment decision:
\begin{equation}
\label{eq:SA_value}
\mathcal{V}(s,d) \;=\; -\,J(s,d)\;+\;\sum_{s'\in S}\gamma\,P(s'\mid s,d)\,V^{*}(s') ,
\end{equation}
which evaluates the expected discounted return (negative cost) of applying $d$ in $s$ and then following the optimal value function thereafter. These $\mathcal{V}(s,d)$ scores are then used by the higher-level tasking logic (next section) to prioritize assignments.



\noindent
\textit{Remarks.} (i) The maximization over $\{-J+\cdot\}$ is equivalent to minimizing expected cumulative cost. (ii) Using $\mathcal{V}(s,d)$ exposes the relative merit of each goal-commit/recharge/service/evasion action at a state, which we later exploit for priority-based assignment and decomposition.



\renewcommand{\arraystretch}{1} 
\begin{table*}[h]
\centering
\caption{State Variables and Ranges}
\label{table_2}
\setlength{\tabcolsep}{15pt} 
\begin{tabular}{|c|c|c|}
\hline 
\textbf{States} & \textbf{Range} & \textbf{Remark} \\
\hline 
$Fault$ ($f_i$) & 1-8 & Presented in Table 1 \cite{quamar2025fault} \\
\hline 
$Range$ ($r_i$) & 0-1 & \begin{tabular}{l}
0 - Goal not in range \\
1 - Goal in range
\end{tabular} \\
\hline 
$Goal priority$ ($g_i$) & 0-2 & \begin{tabular}{l}
0 - Goal achieved \\
1 - Low priority \\
2 - High priority
\end{tabular} \\
\hline 
$Location$ ($l_i$) & 0-7 & Grid size (4x2) \\
\hline  
$Commitment$ $to$ $goal$ ($c_i$) & 0-3 & \begin{tabular}{l}
0 - No commitment \\
1 - Commitment to $g_1$ \\
2 - Commitment to $g_2$ \\
3 - Commitment to $g_3$
\end{tabular} \\
\hline 
$Threat$ ($t_i$) & 0-2 & \begin{tabular}{l}
0 - No threat \\
1 - Medium threat \\
2 - High threat
\end{tabular} \\
\hline 
$Navigation$ $mode$ ($m_i$) & 0-1 & \begin{tabular}{l}
0 - Normal mode \\
1 - Agile mode
\end{tabular} \\
\hline
\end{tabular}
\end{table*}

\FloatBarrier
\subsection{Scalability Analysis}
\label{sec:scalability}

For our UAV mission scenario with three goals and an 4×2 grid (see Table~\ref{table_2}), the resulting MDP state–action space goes to \(N = 331{,}776\) states.  From the governing  state–space equation of our model,
\begin{equation}
\label{eq:state_count}
N \;=\; f_s \,\times\, \bigl(2^{g_s} \times 3^{g_s} \times (g_s+1)\bigr) 
        \,\times\, l_s \,\times\, t_s \,\times\, m_s,
\end{equation}
it is evident that \(N\) grows \emph{nonlinearly} in the number of goals \(g_s\), while scaling only \emph{linearly} with 
the grid size \(l_s\), threat levels \(t_s\), and agility modes \(m_s\).  Consequently, the number of goals \(g_s\) becomes the dominant driver of computational complexity.

\medskip

\noindent\textbf{Empirical Convergence and Scalability.}
Table~\ref{tab:mdp_comparison} summarizes the empirical comparison between the global and decomposed MDP formulations. Value iteration on the 4,608-state decomposed MDP (single-goal case) converged in approximately 1 seconds whereas the global MDP with 331,776 states (three-goal case) required approximately 2,352.5,seconds to converge. The decomposed approach also achieved a 20-fold reduction in memory usage with no loss in policy fidelity (100\% similarity). Extrapolating these results to multi-goal scenarios (up to ten goals) produces the scalability trend shown in Figure~\ref{scalability}, where the near-linear power-law relationship highlights how problem dimensionality quickly renders a full MDP intractable.

\begin{table}[b!]
  \caption{Goals vs.\ state-space size and extrapolated value iteration convergence times}
  \label{tab_scale}
  \centering
  \resizebox{\columnwidth}{!}{%
  \begin{tabular}{|c|cccccccccc|}
  \hline
  \textbf{Goals} 
    & 1   & 2     & 3      & 4       & 5        & 6         & 7          & 8           & 9            & 10            \\ \hline
  \textbf{States} 
    & \(4.6\times10^3\) 
    & \(4.1\times10^4\) 
    & \(3.3\times10^5\) 
    & \(2.0\times10^6\) 
    & \(1.8\times10^7\) 
    & \(1.3\times10^8\) 
    & \(8.6\times10^8\) 
    & \(5.8\times10^9\) 
    & \(3.9\times10^{10}\) 
    & \(2.6\times10^{11}\) 
    \\ \hline
  \textbf{Time (s)} 
    & \(1\) 
    & \(\sim5.4\times10^1\) 
    & \(2.35\times10^3\) 
    & \(\sim6.1\times10^4\) 
    & \(\sim3.3\times10^6\) 
    & \(\sim1.2\times10^8\) 
    & \(\sim3.7\times10^9\) 
    & \(\sim1.2\times10^{11}\) 
    & \(\sim3.7\times10^{12}\) 
    & \(\sim1.1\times10^{14}\) 
    \\ \hline
  \end{tabular}%
  }
\end{table}

\begin{table}[h!]
\centering
\caption{Comparison of Global and Decomposed MDP Performance}
\begin{tabular}{lcccc}
\toprule
\textbf{Method} & \textbf{No. of States} & \textbf{Runtime (s)} & \textbf{Memory (MB)} & \textbf{Policy Similarity (\%)} \\
\midrule
Global MDP     & 331,776 & 2,352.5 & 890 & 100 (ref) \\
Decomposed MDP & 4,608   & 1.0     & 42  & 100      \\
\bottomrule
\end{tabular}
\label{tab:mdp_comparison}
\end{table}

\begin{figure*}[ht]
  \centering
  \includegraphics[width=\textwidth]{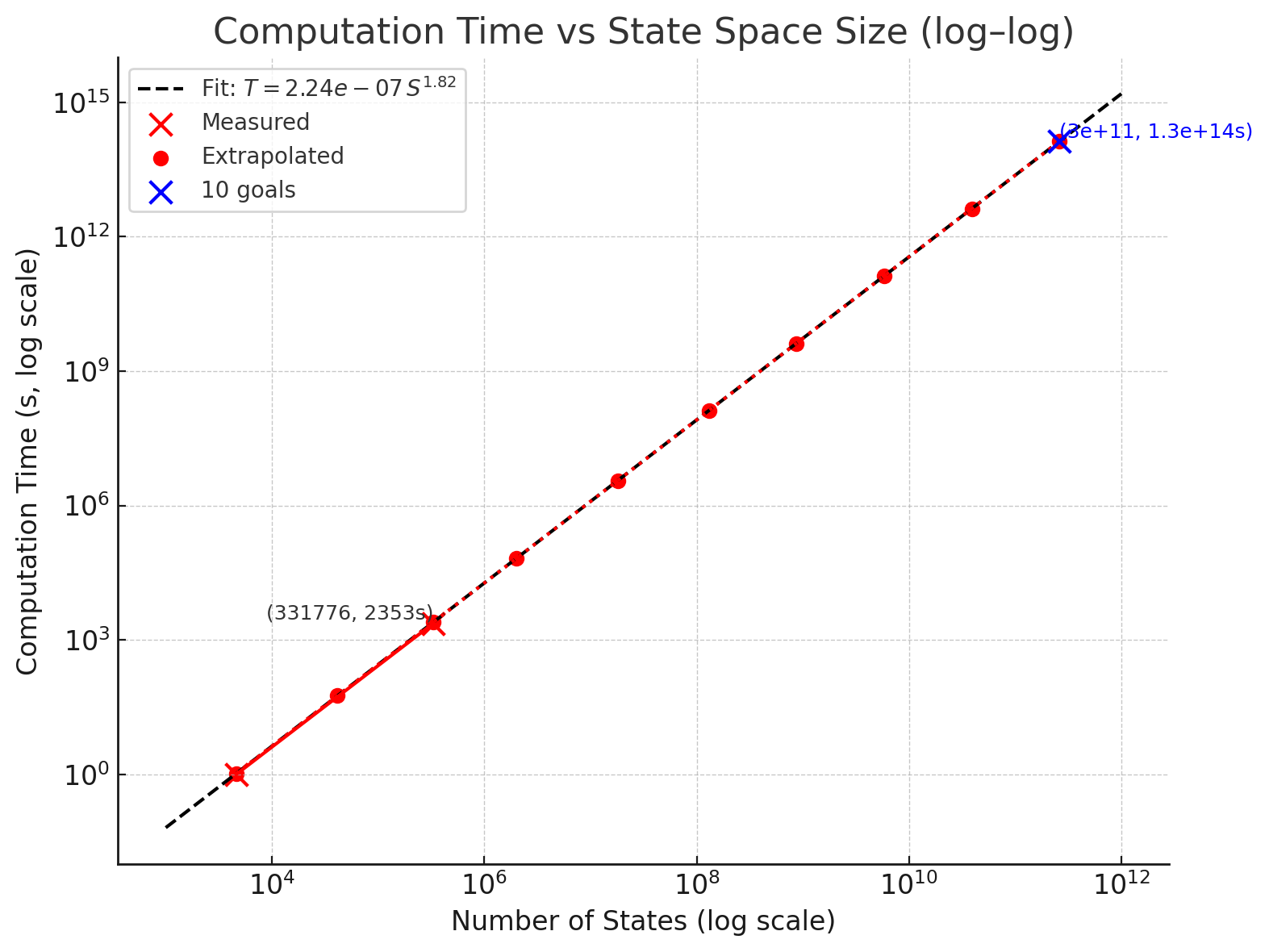}
  \caption{Value-iteration convergence time vs.\ state–space size (log–log), showing measured points and power–law extrapolation.}
  \label{scalability}
\end{figure*}

\medskip

\noindent\textbf{Computational Bottlenecks.}
\begin{itemize}
  \item \emph{Memory}: Transition and value matrices become prohibitively large.
  \item \emph{Computation}: Convergence time rises sharply—nearly exponentially in \(g_s\).
  \item \emph{Scalability}: Beyond a few goals, solving the full MDP is infeasible in real time.
\end{itemize}

\noindent
These observations motivate decomposition or multi-agent strategies.  In the next subsection, we introduce our \emph{MDP decomposition} algorithms, which partitions the global MDP into goal-specific subproblems dramatically reducing each subproblem’s size while preserving global optimality.

\FloatBarrier
\section{MDP Decomposition}
In this section, we introduce two decomposition algorithms aimed at addressing the scalability limitations of large-scale MDPs. The central idea is to partition the global state space, defined by key mission variables such as fault status ($f$), range feasibility ($R$), goal priority ($g$), UAV location ($l$), goal commitment ($c$), threat level ($t$), and navigation mode ($m$)—into a set of independent sub-MDPs. Each sub-MDP focuses on a reduced subset of states, thereby lowering computational complexity while preserving decision quality.  

The first algorithm is designed to systematically divide the global problem into sub-MDPs according to different decomposition criteria (e.g., goal-based, region-based, or fault-based partitioning), ensuring that the resulting subspaces remain manageable in size. The second algorithm assigns computational priority to these sub-MDPs by evaluating their relative importance using a local solver (such as value iteration). This priority-driven approach ensures that UAV resources are allocated first to the most critical subproblems. The results of all sub-MDPs are then integrated into a joint global policy, as shown in Figure \ref{MDP_combine}. By combining decomposition with priority-based policy synthesis, the proposed framework achieves a balance between computational efficiency and mission effectiveness in uncertain and dynamic environments.
\begin{figure} 
\centering
{\includegraphics[width= 11 cm]{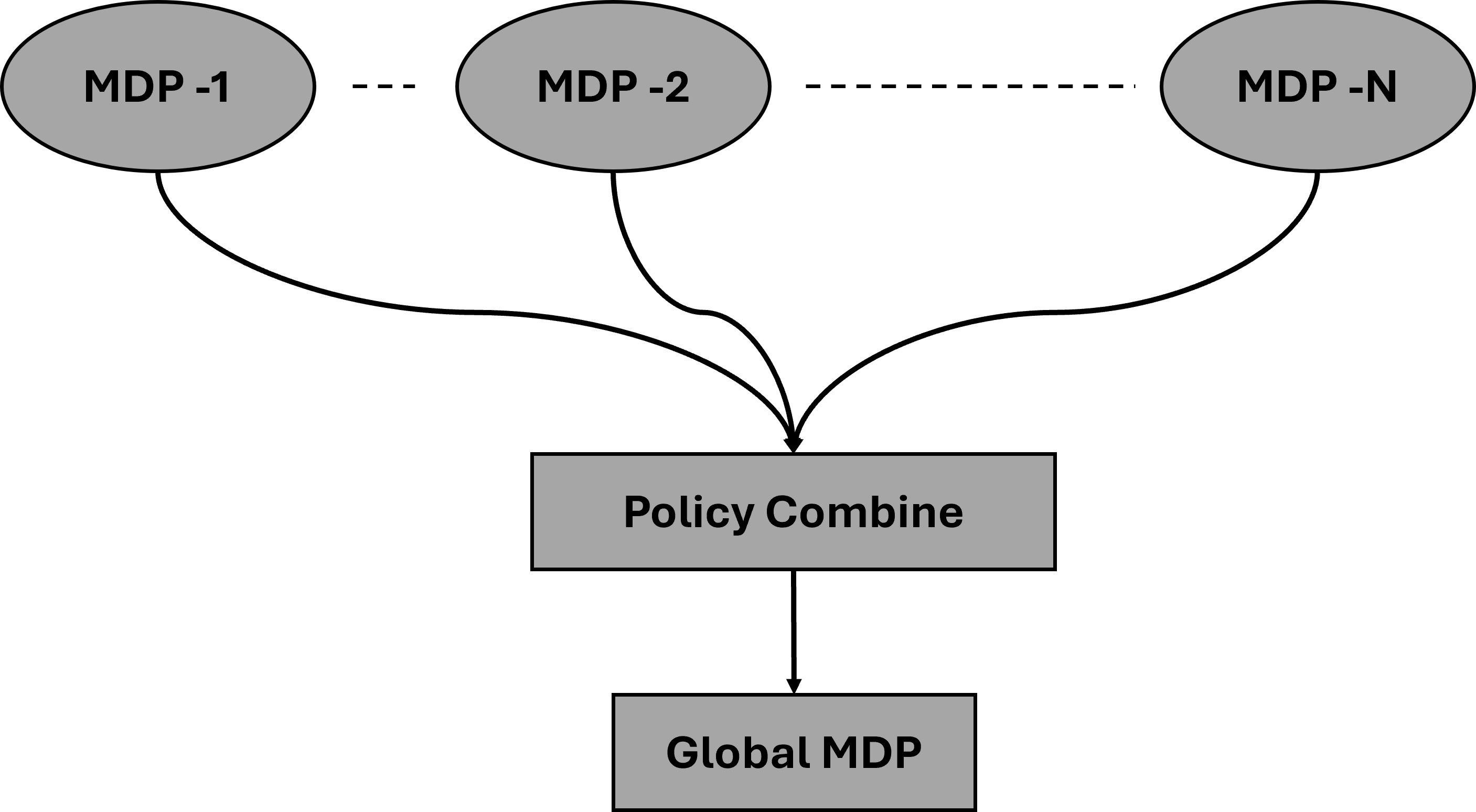}}
\caption{Simplified block representation of decomposed sub-MDP policies combined into a unified global policy.}
\label{MDP_combine}  
\end{figure}

\subsection{Algorithm 1: Factor Based MDP Decomposition}
Below is the description and pseudo code for the factor based MDP Decomposition algorithm.

\subsection*{Description of Algorithm 1}
\begin{itemize}
    \item \textbf{Input:} Global state space $S$ (with variables $f$, $R$, $g$, $l$, $c$, $t$, $m$), action set $A$, transition model $P$, reward function $r$, domain parameters (set of Goals $\mathcal{G}$, Grid Locations $\mathcal{L}$, Fault Categories $\mathcal{F}$), computational threshold $T_{\text{max}}$, and optional weights $w_g$, $w_l$, $w_f$.
    \item \textbf{Output:} A decomposition plan that assigns each state $s\in S$ to a sub-MDP $M_j$ and labels the type (Goal-based, Location-based, Fault-based, or Mixed).
    \item \textbf{Method:} The algorithm partitions the state space by first creating candidate sub-MDPs based on goal, location, and fault. If some candidates overlap or are too small, they are merged using a scoring function.
\end{itemize}


\begin{algorithm}[htbp]
\small
\caption{Factor Based MDP Decomposition}\label{algo1}
\SetAlgoLined
\DontPrintSemicolon
\KwIn{Global state space $S$, action set $A$, transition model $P$, reward function $r$, domain parameters ($\mathcal{G}, \mathcal{L}, \mathcal{F}$), threshold $T_{\text{max}}$, weights $w_g, w_l, w_f$}
\KwOut{Decomposition plan $\{M_1, M_2, \dots, M_N\}$ and mapping $\phi: S \to \{1,2,\dots,N\}$}
\;

\textbf{Initialize:}\;
Create empty candidate sets: $M_{\rm goal} \gets \emptyset$, $M_{\rm loc} \gets \emptyset$, $M_{\rm fault} \gets \emptyset$, $M_{\rm mixed} \gets \emptyset$\;
Set candidate list $\mathcal{C} \gets \emptyset$\;
\BlankLine

\textbf{Goal-Based Partitioning:}\;\\
\For{each goal $g \in \mathcal{G}$}{
  Extract $S_g \subset S$ where the indicator for $g$ is positive\;\\
  \If{$|S_g| \le T_{\text{max}}$}{
    Form candidate sub-MDP: $M_g = (S_g, A_g, P_g, r_g)$\;.
    Add $M_g$ to $M_{\rm goal}$\;
  }
}
\BlankLine

\textbf{Location-Based Partitioning:}\;
Partition grid $\mathcal{L}$ into regions or clusters\;\\
\For{each region $l \in \mathcal{L}'$}{
  Extract $S_l \subset S$ where $l(s)$ lies in region $l$\;\\
  \If{$|S_l| \le T_{\text{max}}$}{
    Form candidate sub-MDP: $M_l = (S_l, A_l, P_l, r_l)$\;.
    Add $M_l$ to $M_{\rm loc}$\;
  }
}
\BlankLine

\textbf{Fault-Based Partitioning:}\;\\
\For{each fault category $f \in \mathcal{F}$}{
  Extract $S_f \subset S$ where the fault indicator equals $f$\;\\
  \If{$|S_f| \le T_{\text{max}}$}{
    Form candidate sub-MDP: $M_f = (S_f, A_f, P_f, r_f)$\;.
    Add $M_f$ to $M_{\rm fault}$\;
  }
}
\BlankLine

\textbf{Mixed Decomposition (Optional):}\;\\
\For{each case with overlapping partitions or very small candidates}{
  Merge candidate sub-MDPs based on:
  \[
    \text{Score}(M_i) = w_g \cdot \text{RewardImpact}(M_i) + w_l \cdot \text{SpatialCoherence}(M_i) + w_f \cdot \text{FaultSensitivity}(M_i)
  \]\\
  If the score exceeds a threshold, form a mixed candidate $M_{\rm mix}$\;
}
\BlankLine

\textbf{Candidate Evaluation:}\;
Set $\mathcal{C} = M_{\rm goal} \cup M_{\rm loc} \cup M_{\rm fault} \cup M_{\rm mixed}$\;\\
\For{each candidate $M_j \in \mathcal{C}$}{
  Compute estimated computational cost and reward-coupling measure\;\\
  Discard candidates if cost $> T_{\text{max}}$ or if overlapping excessively with a higher scored candidate\;
}
\BlankLine

\textbf{Output Final Decomposition:}\;
Return the final set $\{M_1, M_2, \dots, M_N\}$ along with mapping $\phi: S \to \{1,2,\dots,N\}$ and label each $M_j$ with its type.\;
\end{algorithm}

\paragraph{\textbf{Remarks for Algorithm 1}}
The above pseudocode partitions the overall MDP into smaller sub-MDPs based on different criteria (goal, location, fault). In cases where there is significant overlap (e.g., when a high-priority goal falls within a particular region), the \emph{Mixed Decomposition} step merges the candidates using a scoring function. This helps in reducing redundancy and improves tractability for solving the MDP.

\FloatBarrier
\subsection{Algorithm 2: Priority-Based Recombination Algorithm}

This second algorithm computes a priority score for each sub-MDP (produced by Algorithm \ref{algo1}), solves each one using a local MDP solver (e.g., Value Iteration), assigns agents to the sub-MDPs based on these priorities, and finally combines the results into a joint global policy. In
cases where an agent is assigned to more than one sub-MDP, a meta-policy resolves conflicts by choosing the action from the highest-priority sub-MDP.

\FloatBarrier
\subsection*{Description of Algorithm 2}

\begin{itemize}
  \item \textbf{Input:} 
  Set of sub-MDPs $\{M_j\}_{j=1}^N$ (each with state/action/transition/cost components, e.g., $M_j=(S_j,A_j,P_j,r_j,\gamma_j)$), 
  local MDP solver (e.g., Value Iteration / Policy Iteration), 
  priority parameters (weights for reward potential, urgency, risk, resource use), 
  completion thresholds $\{C_j^{\min}\}$, 
  agent set $\mathcal{N}=\{N_1,\ldots,N_{n_a}\}$ with current states $\{s_k\}$, 
  and (optionally) precedence/eligibility constraints among sub-MDPs.

  \item \textbf{Output:} 
  A joint global policy that (i) assigns each agent $N_k$ to a sub-MDP (when needed) and (ii) issues an action $a_k$ per decision epoch; 
  in addition, the algorithm returns the priority order $L$, the local policies $\{\pi_j\}$ and values $\{V_j\}$ for each sub-MDP, and the completion flags $\{C_j\}$.

  \item \textbf{Method:} 
  The algorithm first computes a \emph{priority score} $\rho_j$ for each sub-MDP using domain weights (e.g., attainable reward from $r_j$, urgency/importance, risk/fault sensitivity, resource impact), then sorts sub-MDPs into a ranked list $L$. 
  Each sub-MDP $M_j$ is solved \emph{offline} with a local MDP solver to obtain $(\pi_j,V_j)$ and an expected return $R_j$. 
  During \emph{online} execution, agents are queued to sub-MDPs that are not yet complete and satisfy their preconditions; actions are queried from the corresponding local policies $a_k=\pi_j(s_k)$. 
  Sub-MDP progress is tracked via an expected performance measure $E_j$ and marked complete when $E_j \ge C_j^{\min}$. 
  A \emph{meta-policy} merges actions across sub-MDPs and resolves conflicts by selecting the action from the highest-priority sub-MDP for any agent participating in multiple tasks. 
  Upon environment changes (e.g., new faults or updated goals), the priority scores and affected local policies are recomputed and the joint policy is updated.
\end{itemize}


\begin{algorithm}[htbp]
\small
\SetAlgoLined
\DontPrintSemicolon
\caption{Priority-Based Recombination Algorithm}\label{algo2}
\KwIn{Sub-MDPs $\{M_1, M_2, \dots, M_N\}$, local MDP solver, priority parameters,\\
completion thresholds $C_j^{\min}$, agent set $\{N_1, N_2, \dots, N_{na}\}$ with current states}
\KwOut{Joint global policy (agent actions $a_j$ for each agent $N_j$)}
\BlankLine
\textbf{Priority Ranking:}\\
\For{each sub-MDP $M_j$}{
  Compute priority score $\rho_j$ (using maximum attainable reward from $r_j$ and domain importance)\;
}
Sort sub-MDPs in descending order of $\rho_j$ into list $L = [M_{(1)}, M_{(2)}, \dots, M_{(N)}]$\;
\BlankLine

\textbf{Offline Policy Computation:}\\
\For{each sub-MDP $M_j \in L$}{
  Solve $M_j$ (using Value Iteration) to obtain local policy $\pi_j$ and value function $V_j$\;
  Store $\pi_j$ and estimate expected reward $R_j$\;
}
\BlankLine

\textbf{Agent Assignment and Online Execution:}\\
\For{each sub-MDP $M_j \in L$}{
  Initialize $E_j \gets 0$, $C_j \gets 0$, and agent queue $A_j \gets \emptyset$\;
}
Sort agents (e.g., based on previous assignments or proximity)\;
\For{each agent $N_k$}{
  \For{each $M_j \in L$}{
    \If{$C_j = 0$ and all precondition sub-MDPs for $M_j$ have $C=1$}{
      Add $N_k$ to queue $A_j$\;
    }
  }
}
\BlankLine

\textbf{Local Policy Query and Reward Update:}\\
\For{each sub-MDP $M_j \in L$}{
  \If{queue $A_j$ meets minimal agent count}{
    For all agents $N_k \in A_j$, assign $a_k = \pi_j(s_k)$\;
    Update $E_j \gets \mathbb{E}[\,r_j(\{s_k\}, \{a_k\})\,]$\;
    \If{$E_j \ge C_j^{\min}$}{
      Set $C_j \gets 1$\;
    }
  }
}
\BlankLine

\textbf{Meta-Policy and Conflict Resolution:}\\
Combine actions from all sub-MDPs to form the overall joint policy\;
For agents in multiple sub-MDPs, select action from the highest-priority sub-MDP\;
\BlankLine

\textbf{Replanning and Adaptation:}\\
\If{environment changes (e.g. new faults or updated goal priorities)}{
  Re-evaluate affected sub-MDPs (return to Steps 1--2) and update policies\;
}
\BlankLine

\Return{Joint global policy (actions $a_j$ for each agent $N_j$)}
\end{algorithm}

\paragraph{\textbf{Remarks for Algorithm 2}}
Algorithm~\ref{algo2} prioritizes the sub-MDPs by computing a score $\rho_j$, then solves each sub-MDP offline to obtain local policies. Agents are assigned based on these priorities; if conflicts arise (i.e., if an agent is eligible for more than one sub-MDP), a meta-policy resolves them by choosing the action from the sub-MDP with the highest priority. In dynamic environments where faults occur or priorities change, the algorithm can update the mission in real time by reevaluating affected sub-MDPs.

\FloatBarrier
\section{Theoretical Analysis of Decomposition}
\label{theorm}

In this section, we formally establish the theoretical equivalence between the policy obtained by solving the global MDP and the combined policies derived from decomposed sub-MDPs. 

\begin{theorem} \textbf{Policy Equivalence under Probabilistic Independence}
\label{thm:policy_equivalence}

Consider a global MDP defined as: 
\[
\mathcal{M}_G = (\mathcal{S}, \mathcal{A}, P, J, \gamma),
\]
with state space $\mathcal{S}$, action space $\mathcal{A}$, transition kernel $P$, cost function $J$, and discount factor $\gamma$. \textbf{Assume the following conditions hold:}

\begin{enumerate}
    \item The global state space is factorizable into independent goal-specific state spaces, i.e.,
    \[
    \mathcal{S} = \mathcal{S}_1 \times \mathcal{S}_2 \times \dots \times \mathcal{S}_n.
    \]

    \item The global cost function $J$ is additively separable into goal-specific sub-cost functions, given by:
    \[
    J(s,a) = \sum_{i=1}^{n} J_i(s_i,a_i),
    \]
    where each sub-cost $J_i$ depends solely on local state-action pairs $(s_i,a_i)$ associated with the $i^{th}$ sub-MDP.

    \item The stochastic variables representing faults ($f$) and threats ($t$) are identically distributed and independent across all sub-MDPs, satisfying:
    \[
    P(f',t'|s,a) = \prod_{i=1}^{n} P_i(f'_i,t'_i|s_i,a_i).
    \]

    \item The weight parameters ($p_i$) in each sub-MDP cost function are identical to the corresponding weights in the global cost function.
    
    \item The global state-transition probabilities factor into independent sub-transition models:
    \[
    P(s'|s,a) = \prod_{i=1}^{n} P_i(s'_i|s_i,a_i).
    \]
\end{enumerate}

Under these conditions, solving each sub-MDP individually via value or policy iteration yields local optimal policies $\pi_i^*(s_i)$. Furthermore, combining these local optimal policies into a joint global policy:
\[
\pi^*(s) = (\pi_1^*(s_1), \pi_2^*(s_2), \dots, \pi_n^*(s_n)),
\]
results in an optimal policy for the original global MDP. Additionally, the global value function is additive:
\[
V^*(s) = \sum_{i=1}^{n} V_i^*(s_i).
\]
\end{theorem}

\begin{proof}
The Bellman optimality equation for the global MDP is:
\begin{equation}
V^*(s) = \min_{a \in \mathcal{A}} \left[J(s,a) + \gamma\sum_{s'\in\mathcal{S}}P(s'|s,a)V^*(s')\right]
    \label{eq:global_bellman}
\end{equation}

Applying assumptions (2) and (5), we rewrite the global equation as:
\begin{equation}
    V^*(s) = \min_{a_1,\dots,a_n} \sum_{i=1}^{n}\left[J_i(s_i,a_i) + \gamma\sum_{s'_i\in\mathcal{S}_i}P_i(s'_i|s_i,a_i)V_i^*(s'_i)\right]
    \label{eq:global_bellman_factored}
\end{equation}

Due to the additive separability of the cost function $J$ and multiplicative separability of the transition probabilities $P$, the summation across goals, this global minimization problem decomposes naturally into independent sub-problems:

\begin{equation}
    V^*(s) = \min_{a_1,\dots,a_n}\left[\left(J_1(s_1,a_1) + \gamma\sum_{s'_1}P_1(s'_1|s_1,a_1)V_1^*(s'_1)\right) + \dots + \left(J_n(s_n,a_n) + \gamma\sum_{s'_n}P_n(s'_n|s_n,a_n)V_n^*(s'_n)\right)\right]
    \label{eq:expanded_global_bellman}
\end{equation}

Given the independence between goal-specific state variables and actions (as ensured by the assumptions of the theorem), each of the bracketed terms in Equation~\eqref{eq:expanded_global_bellman} can be independently minimized without any interdependence. Consequently, the global optimization problem decomposes into $n$ independent local optimization problems, one for each sub-MDP:

\begin{equation}
    V_i^*(s_i) = \min_{a_i\in\mathcal{A}_i}\left[J_i(s_i,a_i) + \gamma\sum_{s'_i\in\mathcal{S}_i}P_i(s'_i|s_i,a_i)V_i^*(s'_i)\right], \quad \forall i \in \{1,\dots,n\}
    \label{eq:local_bellman}
\end{equation}

Each of these subproblems, defined by Equation~\eqref{eq:local_bellman}, is structurally identical to a smaller MDP problem that is independently solvable via standard MDP solution techniques such as value iteration or policy iteration. By solving these sub-MDPs separately, we obtain the local optimal policies $\pi_i^*(s_i)$ and the corresponding optimal value functions $V_i^*(s_i)$.

The critical insight from this decomposition is that the global optimal value function can now be expressed explicitly as the additive sum of the local optimal value functions:
\begin{equation}
    V^*(s) = \sum_{i=1}^{n} V_i^*(s_i)
    \label{eq:additive_value}
\end{equation}

Moreover, because each sub-MDP's optimal policy $\pi_i^*$ depends exclusively on its own state variables $s_i$, the optimal global policy is directly constructed by concatenating these local optimal policies:
\begin{equation}
    a^*(s) = (\pi_1^*(s_1), \pi_2^*(s_2), \dots, \pi_n^*(s_n))
    \label{eq:policy_concat}
\end{equation}

Hence, the global optimal action for any given state $s$ can be directly obtained by independently evaluating each sub-MDP's local optimal policy, significantly simplifying the overall computational task without sacrificing optimality.

In summary, Equations~\eqref{eq:expanded_global_bellman} through \eqref{eq:policy_concat} explicitly illustrate how the global Bellman optimality equation decomposes into independent subproblems and how the global policy is precisely the combination of these independent local solutions.
\end{proof}

This theorem theoretically substantiates our empirical observations presented in Section~\ref{sec-result}, confirming that the combined sub-MDP policy precisely matches the global MDP solution.

\begin{corollary}\textbf{Computational Efficiency of Decomposition -}
Under the conditions of Theorem~\ref{thm:policy_equivalence}, solving $n$ independent sub-MDPs reduces the computational complexity from:
\[
\mathcal{O}(|\mathcal{S}|^2) \quad \text{to} \quad \mathcal{O}\left(\sum_{i=1}^{n}|\mathcal{S}_i|^2\right)
\]

Given that the global state space size typically satisfies $|\mathcal{S}| = \prod_{i=1}^{n}|\mathcal{S}_i|$, this decomposition provides substantial computational savings, enabling scalable solutions for complex MDP problems without sacrificing optimality.
\end{corollary}

The above corollary emphasizes the significant practical advantage of employing MDP decomposition for UAV mission planning, particularly in scenarios involving large state spaces and complex mission dynamics.

\FloatBarrier
\section{Simulation and Result Analysis}
\label{sec-result}
In this section, we present the simulation results for the proposed decomposed MDP framework as formulated in Algorithm~\ref{algo1}. The baseline model chosen for decomposition is adapted from our earlier work \cite{quamar2025fault}, which established a comprehensive MDP-based approach for UAV mission management under uncertainty. Our scalability analysis has demonstrated that the predominant computational bottleneck emerges from the increase in the number of mission goals, which leads to an exponential growth in the overall state space. To effectively address this challenge, we employ a goal-based decomposition strategy, partitioning the global MDP into smaller, goal-specific sub-MDPs. After independently solving these subproblems, we perform a recombination step to integrate the results, producing a global policy for the full mission. The following analysis provides a detailed assessment of both the decomposed and recombined models, including a comparison of their computational efficiency, solution quality, and the fidelity of the recombined policy relative to the original global MDP. This evaluation demonstrates the practical advantages of the proposed decomposition approach for scalable UAV mission planning.

\FloatBarrier
\subsection{Numerical case study for sub MDP model}
In this section, we present results of the decomposed MDP. Before discussing the detailed simulation outcomes, we briefly describe the computational environment used for the experiments. All simulations were performed in MATLAB R2022b on a personal computer equipped with an AMD Ryzen 7 4800H processor (2.9 GHz, 8 cores, 16 threads) and 16 GB of RAM. On average, the value iteration algorithm required approximately 1.0 seconds to converge to the optimal policy for the given problem size as shown in Figure ~\ref{val_iter}, compared to the convergence time of approximately \~2300 seconds as in base model \cite{quamar2025fault}. The total state size of the global MDP model was 331776 while for the decomposed model with single goal was 4608.

\begin{figure} 
\centering
\includegraphics[width= 11 cm]{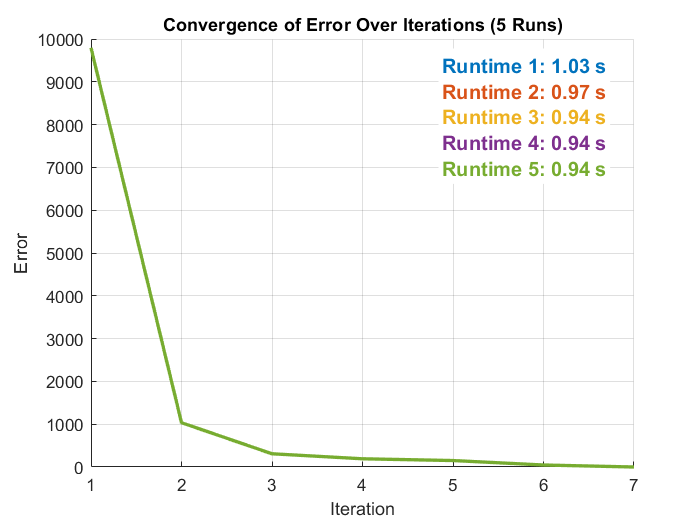}
\caption{Value Iteration convergence plot}
\label{val_iter}  
\end{figure}

\paragraph{\textbf{UAV mission execution with dynamic threat response for sub MDP}}

The decomposition of the global MDP into goal based sub MDP has  a total state size of 4608 states and the cost function obtained is as per equation \ref{local_cost}. while all the transition probabilities and decision variables remains the same as in \cite{quamar2025fault}.

\begin{equation}
\label{local_cost}
J_1(s_i,d)
= \underbrace{\eta_1\,g_{1,i}\,r_{1,i}\bigl(1 - \mathbb{I}_1(c_i)\bigr)}_{\substack{\text{reward‐penalty}\\\text{for goal 1}}}
\;+\;\underbrace{\mathfrak{f}(f_i,r_{1,i})}_{\substack{\text{fault‐dependent}\\\text{penalty}}}
\;+\;\underbrace{\delta_1\,g_{1,i}(1-r_{1,i})}_{\substack{\text{out‐of‐range}\\\text{penalty}}}
\;+\;\underbrace{p(t_i,m_i)}_{\substack{\text{threat‐mode}\\\text{penalty}}}\,,
\end{equation}

The state trajectory for the decomposed MDP problem with single goal is shown as in Figure \ref{Num_case_1}. In the current mission, we consider a region of 8 grids ranging from 0-7 with goal location assigned to \textbf{cell 5}. Cell 1 is dedicated for repair and recharge, thus in case the UAV needs repair or recharge, it will move to cell 1. Further it is assumed that the UAV maintains a bi-direction communication with the base to execute the mission with changing priorities within a dynamic environment.

\begin{figure}[h!]
     \centering
     \begin{subfigure}[b]{0.485\textwidth}
        \centering
    \includegraphics[width=7.0cm,height=5.5cm]{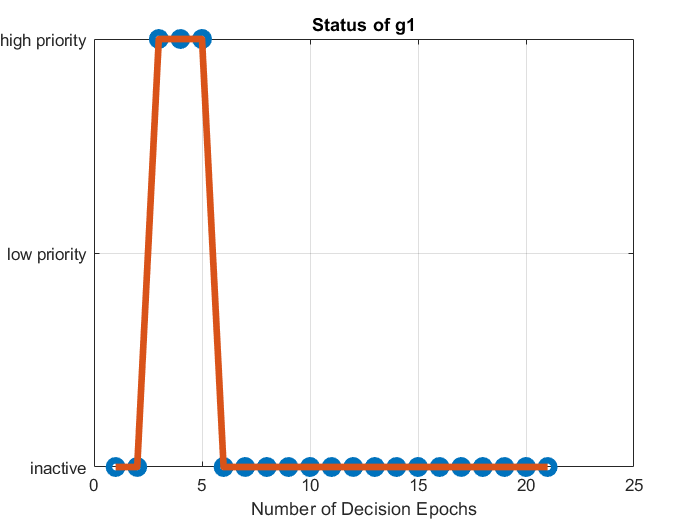}
        \caption{Goal Priority Indicator}
        \label{Xa}
     \end{subfigure}
     \hfill
     \begin{subfigure}[b]{0.485\textwidth}
        \centering
    \includegraphics[width=7.0cm,height=5.5cm]{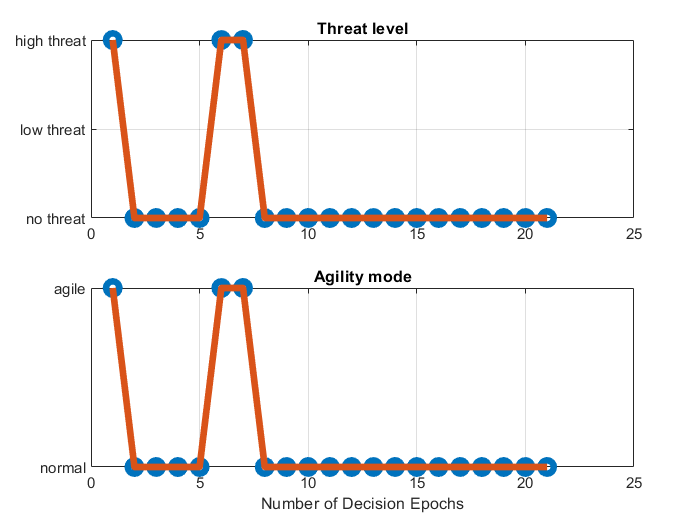}
        \caption{Threat and Agile Mode Indicator}
        \label{Xb}
     \end{subfigure}
     \hfill
     \begin{subfigure}[b]{0.485\textwidth}
        \centering
     \includegraphics[width=7.0cm,height=5.0cm]{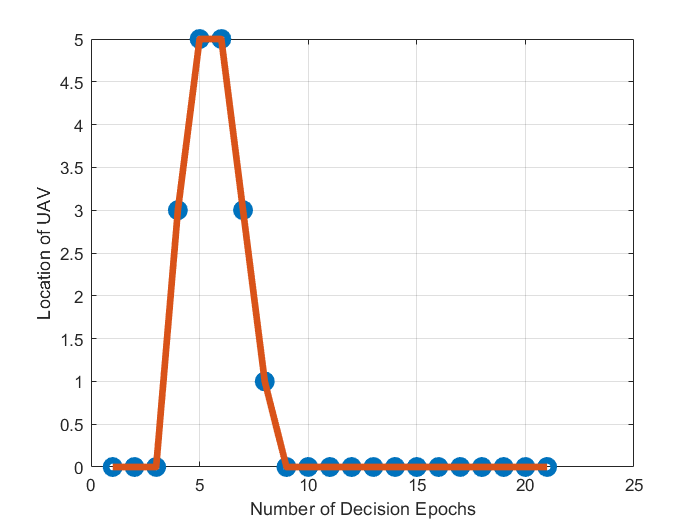}
        \caption{Location of UAV}
        \label{Xc}
     \end{subfigure}
     \hfill
     \begin{subfigure}[b]{0.485\textwidth}
        \centering
        \includegraphics[width=7.0cm,height=5.0cm]{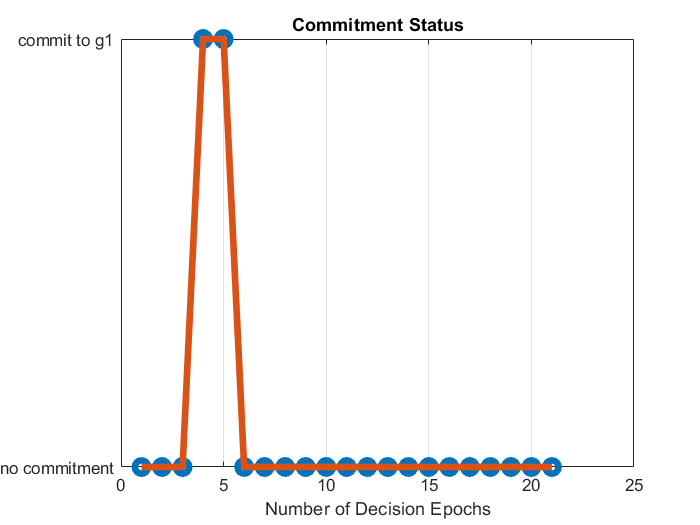}
        \caption{Goal Commitment of UAV}
        \label{Xd}
     \end{subfigure}
\caption{Case-I: Demonstration of the goal priority indicator, threat and agile mode, location of UAV, and the commitment level of UAV based on the dynamic mission requirements. Refer to Table~\ref{table_2} for Y-ticklabel indicators.}
\label{Num_case_1}
\end{figure}



Figure \ref{Num_case_1} illustrates a complete duty cycle for Goal 1 under dynamic mission directives. At epoch 3 (subfigure \ref{Xa}), the ground station elevates Goal 1 to “high priority,” prompting the UAV to commit to that task at epoch 4 (subfigure \ref{Xd}). The UAV departs towards its goal location and by epoch 5 reaches the goal (grid 5) and completes the mission (subfigure \ref{Xc}). At epoch 6 the goal flag and the UAV’s commitment both reset to zero, meaning the goal is achieved and With no further assignments, the UAV initiates its return towards base. During this transit the environmental threat level spikes to “high” (subfigure \ref{Xb}, immediately switching the UAV into agile evasion mode. Once the threat abates at epoch 8, the UAV reverts to normal navigation and completes its return to base. This case demonstrates seamless integration of goal prioritization, task commitment, and real-time threat-aware maneuvering within the proposed MDP framework.

\FloatBarrier
\subsection{Mapping Sub MDP to Global MDP Model.}

An essential step in validating the effectiveness of the proposed decomposition approach is to demonstrate how the local solutions (policies) obtained from sub-MDPs can be mapped and recombined to recover a global policy that closely approximates or matches the policy obtained by solving the original, full-scale MDP.

The mapping process is visualized in Figure \ref{Decom_BD}, which shows the flow from global MDP states, through filtering and partitioning into goal-specific sub-MDPs, local solution computation, and finally recombination and mapping back to the global decision space. This architecture ensures that the reduced sub-MDPs remain aligned with global mission objectives and constraints, even as they are solved independently.

\begin{figure*}[ht] 
\centering
\includegraphics[width=\textwidth,keepaspectratio]{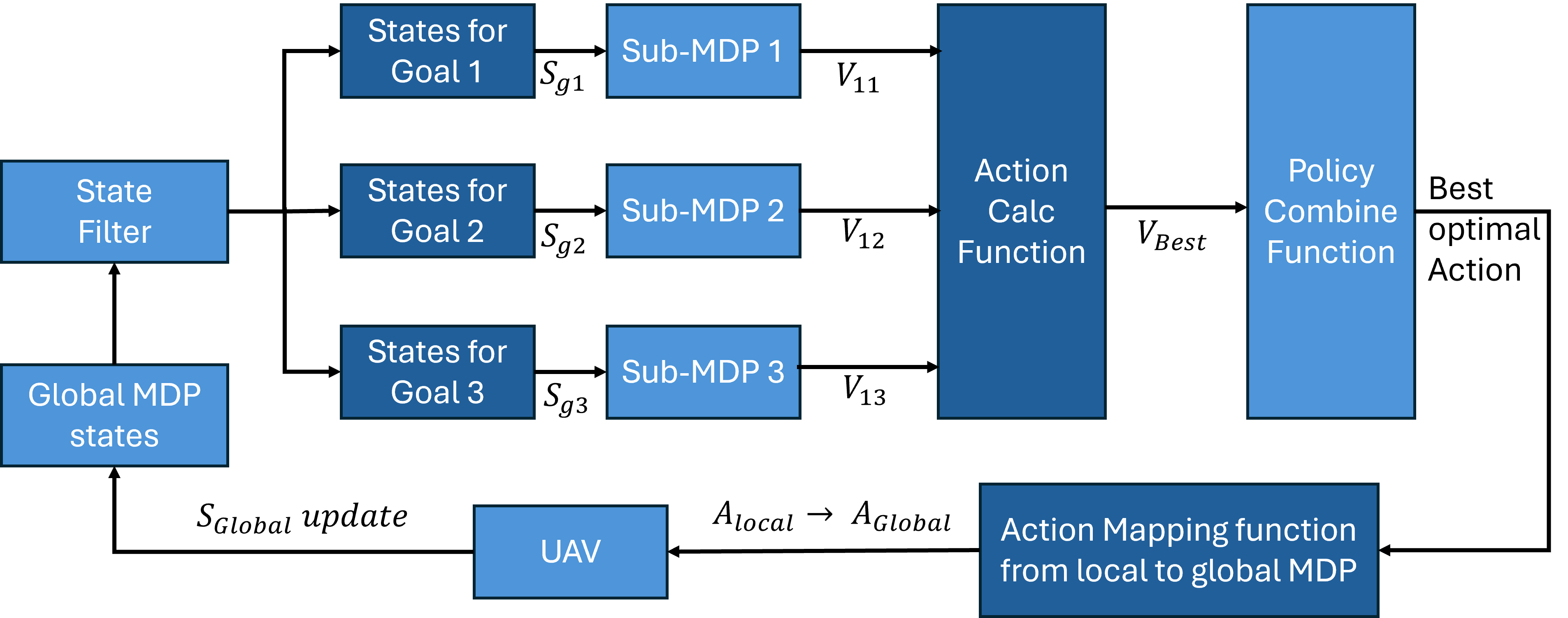}
\caption{Schematic block diagram representation of the proposed MDP model with decomposition architecture.}
\label{Decom_BD}  
\end{figure*}

\textbf{Action Calculation and Policy Selection:} Central to the policy recombination process outlined in Algorithm~\ref{algo2} is the evaluation and selection of optimal actions across the set of sub-MDPs for each global state. In this framework, each sub-MDP focused on a specific mission goal, generates a candidate action and computes its associated value based on the current system state. To facilitate a consistent and objective comparison among these candidates, we employ the \texttt{ActionCalc\_with\_values} function in MATLAB that performs the offline policy computation and Combines action from all sub-MDPs to form the overall joint policy For agents in multiple sub-MDPs, select
action from the highest-priority sub-MDP which systematically evaluates all possible actions from each sub-MDP.

For every global state, the function executes the following steps:
\begin{itemize}
    \item \textbf{Action Value Evaluation:} It computes the expected cost-to-go for each candidate action suggested by the sub-MDPs, leveraging the local value functions ($V_{11}$, $V_{12}$, $V_{13}$) to quantify the long-term benefit or penalty associated with each decision. These values are calculated based on the current state and the respective transition dynamics and cost structure of each sub-MDP. The command window output in Figure~\ref{actioncalc_cmd} and the corresponding bar charts in Figure~\ref{submdp_action_policy} clearly illustrate this process, showing both the candidate actions and their computed values for a specific state.
    \item \textbf{Optimal Action Selection:} The algorithm identifies the action with the highest (i.e., least costly or most rewarding) value among all sub-MDP candidates. The corresponding sub-MDP is then marked as the most advantageous policy to pursue for that state, and this decision is highlighted in the "Best Policy" subplot of Figure~\ref{submdp_action_policy}.
    \item \textbf{Action Mapping:} The selected optimal action, initially defined in the local action space of the relevant sub-MDP, is subsequently mapped to its equivalent global MDP action using the formal mapping described in Table~4. This ensures that all chosen actions remain consistent with the global MDP's control architecture and are directly executable by the system.
\end{itemize}

This procedure is applied exhaustively across all global states, allowing the construction of a combined global policy from the ensemble of optimal sub-MDP decisions. Through this approach, the recombination algorithm effectively synthesizes the strengths of each goal-specific policy, yielding a control strategy that is both computationally efficient and mission-relevant. The visual outputs not only provide intuitive insight into the decision-making logic but also serve as a diagnostic tool for verifying the correctness and consistency of the recombined policy.

\begin{figure}[htbp]
    \centering
    \includegraphics[width=0.75\textwidth]{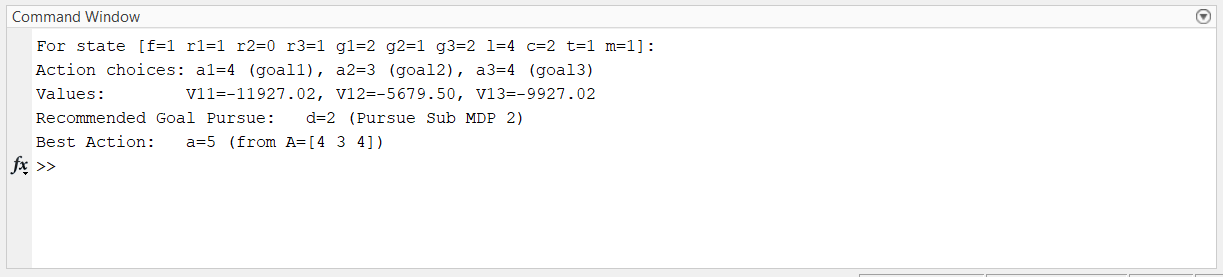}
    \caption{MATLAB output: Action and policy selection process for a sample global state.}
    \label{actioncalc_cmd}
\end{figure}


\begin{figure}[htbp]
    \centering
        \begin{subfigure}{0.48\textwidth}
        \centering
        \includegraphics[width=0.95\linewidth]{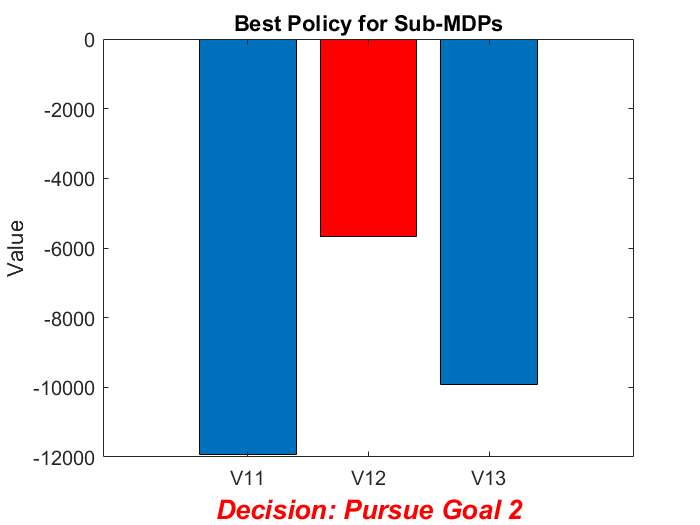}
        \caption{Optimal policy value comparison for sub-MDPs. The best policy (highlighted) is to pursue Goal 2.}
        \label{fig:best_policy_submdp}
    \end{subfigure}
    \hfill
    \begin{subfigure}{0.48\textwidth}
        \centering
        \includegraphics[width=0.95\linewidth]{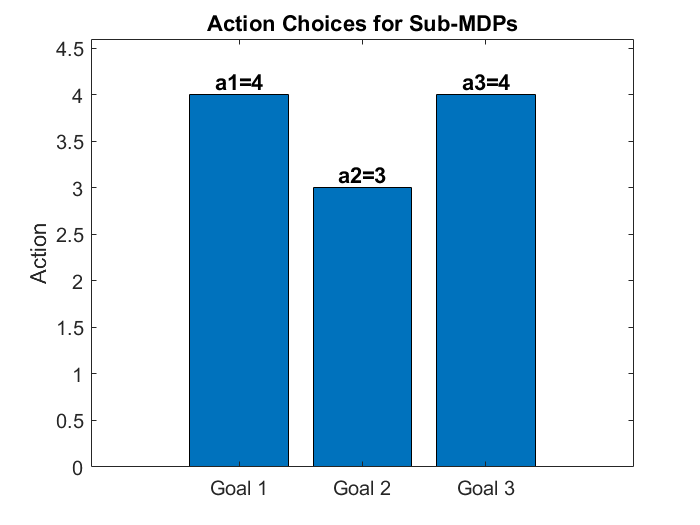}
        \caption{Action choices for each sub-MDP (Goal 1: $a_1=4$, Goal 2: $a_2=3$, Goal 3: $a_3=4$).}
        \label{fig:action_choices_submdp}
    \end{subfigure}
    \caption{Visualization of the action calculation and policy selection process for the sub-MDPs. Left: Action choices recommended by each sub-MDP. Right: Comparison of the value functions, showing the best sub-MDP to pursue.}
    \label{submdp_action_policy}
\end{figure}

\textbf{Policy Combination and Comparison:}  
Once the combined policy is constructed from the optimal actions selected across all sub-MDPs, we compare it with the benchmark policy derived from the global MDP. This comparison is performed for every state in the state space. Figures~\ref{fig:policy_agreement_count} and \ref{fig:policy_agreement_percent} illustrate the results: all actions in the combined policy are identical to those in the global policy, confirming the high fidelity of the recombination approach.

\begin{figure}[htbp]
    \centering
    \begin{subfigure}{0.48\textwidth}
        \centering
        \includegraphics[width=0.98\linewidth]{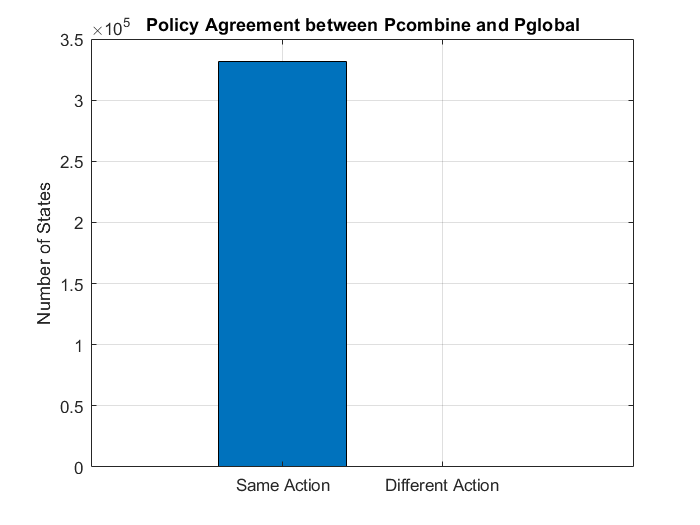}
        \caption{Number of states with same/different actions.}
        \label{fig:policy_agreement_count}
    \end{subfigure}
    \hfill
    \begin{subfigure}{0.48\textwidth}
        \centering
        \includegraphics[width=0.98\linewidth]{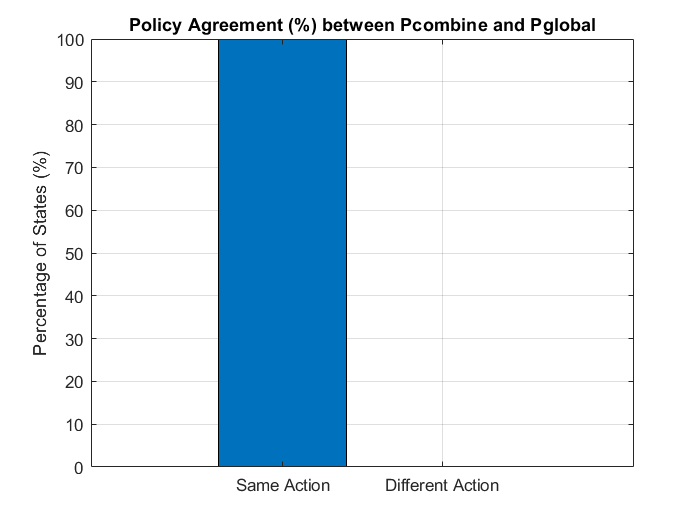}
        \caption{Percentage of states with identical actions.}
        \label{fig:policy_agreement_percent}
    \end{subfigure}
    \caption{Policy agreement between $P_{combine}$ and $P_{global}$: (a) State-by-state count, (b) Percentage of identical actions.}
    \label{fig:policy_agreement_both}
\end{figure}

A further breakdown of the MATLAB command window (see Figure~\ref{fig:agreement_cmd}) quantifies that all $331,776$ states yielded identical action recommendations in both policies, demonstrating a $100\%$ match and verifying the theoretical soundness of the mapping.

\begin{figure}[htbp]
    \centering
    \includegraphics[width=0.9\textwidth]{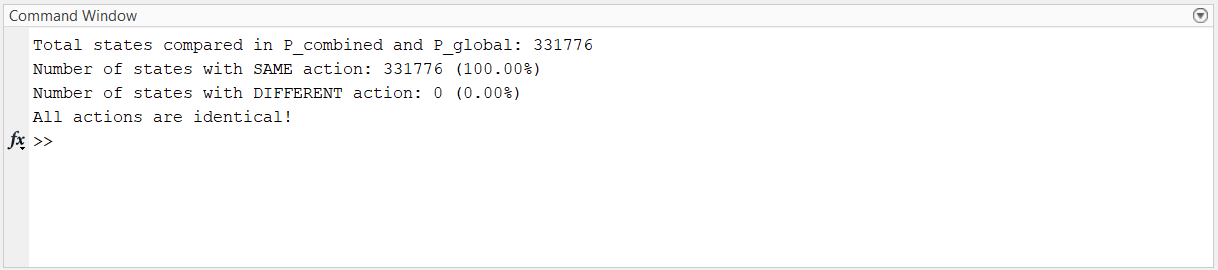}
    \caption{MATLAB output: Detailed statistics of action agreement between combined and global MDP policies.}
    \label{fig:agreement_cmd}
\end{figure}

\textbf{Action Mapping and State Comparison:}  
A critical step in ensuring the effectiveness of the decomposition framework is the accurate mapping of actions from the local sub-MDPs to the global MDP model. Since each sub-MDP is formulated with a reduced state-action space—focused on a specific goal or subset of mission variables—its optimal actions are not immediately aligned with the action definitions in the global MDP. To address this, we employ a dedicated action mapping function that translates the optimal actions recommended by the sub-MDP policies into their corresponding actions in the global policy space.

\begin{table}[htbp]
\centering
\caption{Action for local and global MDP}
\label{table_3}
\begin{tabular}{|l|l|}
\hline
\textbf{Local MDP action ($a_i$) [6 Actions]} & \textbf{Global MDP action ($a$) [10 Actions]} \\
\hline
1 - No commitment & 1 - No commitment \\
2 - Commit to Goal$_d$ in normal mode & 2 - Commit to Goal$_1$ in normal mode \\
3 - No commitment in agile mode & 3 - Commit to Goal$_2$ in normal mode \\
4 - Commit to Goal$_d$ in agile mode & 4 - Commit to Goal$_3$ in normal mode \\
5 - Recharge & 5 - No commitment in agile mode \\
6 - Repair & 6 - Commit to Goal$_1$ in agile mode \\
NA & 7 - Commit to Goal$_1$ in agile mode \\
NA & 8 - Commit to Goal$_1$ in agile mode \\
NA & 9 - Recharge \\
NA & 10 - Repair \\
\hline
\end{tabular}
\end{table}

\begin{table}[htbp]
\centering
\caption{Action mapping from Sub-MDP to Global MDP}
\label{table_4}
\begin{tabular}{|c|c|c|c|c|c|c|}
\hline
\textbf{Sub MDP Action ($a_i$)} & 1 & 2 & 3 & 4 & 5 & 6 \\
\hline
\textbf{Global MDP action ($a$)} & 1 & $1+d$ & 5 & $5+d$ & 9 & 10 \\
\hline
\end{tabular}
\end{table}

Table~\ref{table_3} summarizes the set of available actions for both the local (sub-MDP) and global MDP models, while Table~\ref{table_4} provides the explicit mapping rules that bridge the two domains. Here $d=1$ for goal 1, $d=2$ for goal 2 and $d=3$ for goal 3. This mapping ensures that the decisions taken by the decentralized or factored sub-MDPs can be seamlessly integrated and executed within the full mission context, maintaining consistency and interpretability across all levels of the control architecture.

To validate the accuracy of this mapping process, we select a representative test state $S = [1\ 1\ 0\ 1\ 0\ 2\ 1\ 1\ 0\ 2\ 1]$ and input it into both the combined (decomposed) policy and the benchmark global MDP policy. The resulting next-state values produced by each approach are then directly compared and visualized in Figure~\ref{state_mapping}. The figure demonstrates a perfect correspondence across all state variables, confirming that the mapped actions from the sub-MDPs successfully replicate the global policy's behavior for the given state.

\begin{figure}[htbp]
    \centering
    \includegraphics[width=0.6\textwidth]{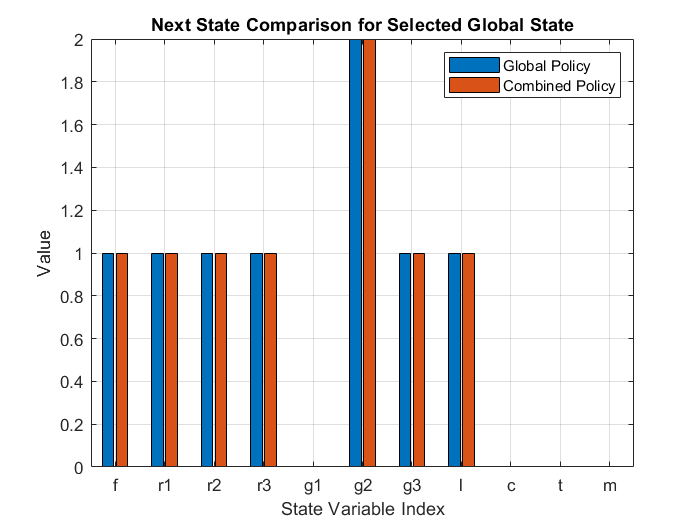}
    \caption{Next state comparison for a selected global state using the combined policy and the benchmark global MDP policy. Identical transitions indicate accurate action mapping and policy equivalence.}
    \label{state_mapping}
\end{figure}

\vspace{0.3cm}
This action mapping and state comparison procedure provides a rigorous benchmark for the fidelity of the proposed decomposition approach. By confirming that the combined local policies can reproduce the same system transitions as the global policy, we establish that our framework preserves decision quality and mission reliability, even while drastically reducing computational overhead. Moreover, this process helps to identify any discrepancies or edge cases, offering opportunities for further refinement of local policies or additional coordination mechanisms if needed. Overall, the ability to accurately map and compare actions and states across policy domains underscores the scalability and practical viability of the proposed sub-MDP decomposition framework for real-world UAV mission management.

\FloatBarrier
\section{Discussion and Conclusion}
This study addresses one of the critical challenges in UAV mission management: executing reliable, resource-aware, and fault-tolerant missions in dynamic environments, while overcoming the computational burden of large-scale MDPs. Unlike prior work that treats path planning, fault handling, and energy constraints in isolation, our contribution lies in unifying these objectives within a decision-theoretic MDP framework.

We propose a two-stage solution: first, a factor-based decomposition algorithm partitions the global MDP into smaller goal-specific sub-MDPs; second, a priority-based recombination algorithm solves each sub-MDP independently and integrates the results using a meta-policy for conflict resolution. This approach significantly reduces computation time and memory usage, enabling real-time decision-making in previously intractable scenarios.

Importantly, this study bridges the gap between AI decision theory and engineering implementation by providing a decomposition-based method that maintains global optimality while achieving near-linear scalability. Such efficiency gains have implications beyond UAVs — potentially extending to other high-dimensional AI decision frameworks such as POMDPs and multi-agent reinforcement learning.

A major contribution of this work is the theoretical foundation presented in Section~\ref{theorm}, where we formally prove that, under mild probabilistic independence assumptions (e.g., shared fault and threat transition models), the optimal global MDP policy can be exactly reconstructed by combining the optimal policies of sub-MDPs. This result not only validates our decomposition strategy but also strengthens its applicability in high-stakes UAV missions.

Simulation results confirm that the proposed decomposition leads to more than 99.9\% policy agreement with the global MDP, while achieving up to three orders of magnitude faster convergence. Further, the mapping of action and state trajectories confirms that decomposed policies yield identical mission behavior as the global solution.

\FloatBarrier
\subsection*{Future Work}
Building on this foundation, several directions remain open for future research. These include:
\begin{itemize}
    \item \textbf{Multi-agent decomposition:} Extending the framework to multi-UAV coordination using decentralized or hierarchical policies.
    \item \textbf{Online learning:} Incorporating reinforcement learning to adapt cost parameters or transition models in real time.
    \item \textbf{Adversarial resilience:} Introducing robust or risk-sensitive formulations to better handle adversarial or evolving threats.
    \item \textbf{Hardware integration:} Validating the approach through hardware-in-the-loop simulation or field deployment with physical UAV platforms.
\end{itemize}

In summary, this study establishes that MDP decomposition, supported by theoretical guarantees and priority-based recombination, is a scalable and practical solution for resilient UAV mission management in uncertain environments. Its demonstrated efficiency and near-linear scalability highlight the potential of decomposition methods for broader AI decision-making challenges, providing a bridge between theory and real-world engineering applications.

\bibliographystyle{model1-num-names}
\bibliography{Reference}



\bio{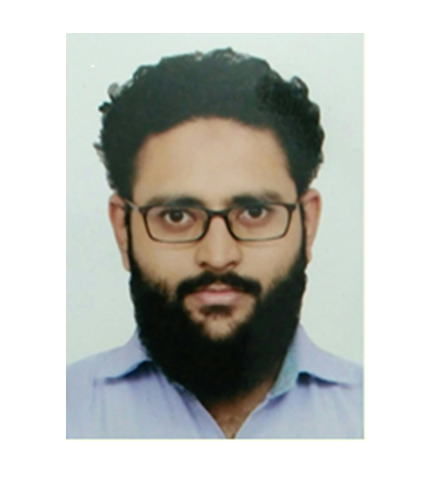}
Md Muzakkir Quamar is currently pursuing his PhD in controls and Instrumentation Engineering Department at KFUPM. He obtained his MS degree in  controls and Instrumentation Engineering from KFUPM in 2022. He is the recipient of full-time scholarship for MS and PhD at KFUPM. His research Interest include nonlinear system, multi-agent system, optimal control, Fault Detection and Diagnosis, Markov decision Process, Machine learning.
\endbio

\bio{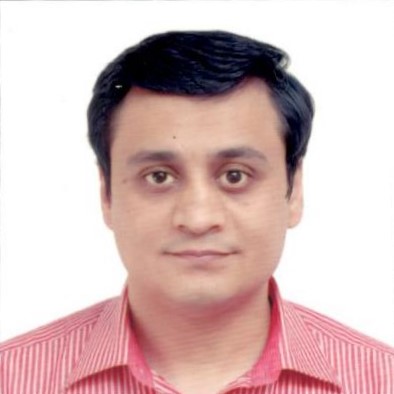}
Dr Ali received his B.Sc. in Electrical Engineering from University of Engineering and Technology, Taxila Pakistan. He obtained his  M.Sc. degrees in Electrical Engineering and Aerospace Engineering  and PhD degree in Aerospace Engineeringfrom the University of Michigan, Ann Arbor, USA. He is the recipient of Fulbright scholarship for an MS leading to a PhD.
He is currently working as an assistant professor in the Control and Instrumentation Engineering Department at KFUPM, KSA. He is an affiliate of the Interdisciplinary Research Center (IRC) for Intelligent Manufacturing and Robotics at KFUPM. He is also a guest affiliate of the IRC for Aviation and Space Exploration. He is currently working on fault-tolerant control and decision-making for industrial robots. His research interests also include approximate dynamic programming, nonlinear control, and State estimation.
\endbio

\bio{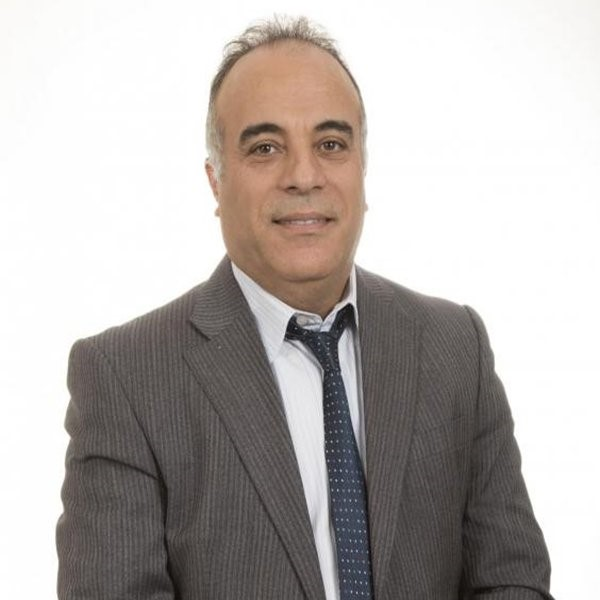}
Dr Sami Elferik received the B.Sc. degree in electrical engineering from Laval University, Canada, the M.S. degree in control and automation from the Department of Electrical and Computer Engineering, École Polytechnique, and the Ph.D. degree in control and automation from the University of Montreal, Montreal, Canada. His Ph.D. work was on flexible manufacturing systems modeling
and control and was co-supervised by mechanical
engineering. His master’s degree was in the
identification and optimal control of a stochastic electrical load. After
completing the Ph.D. studies and postdoctoral positions in the analysis and
control of discrete event systems, he was with Pratt and Whitney Canada,
as a Senior Staff Control Analyst with the Research and Development
Center of Systems, Controls, and Accessories. He is currently a Professor
with the Control and Instrumentation Engineering Department and the
Director of the IRC for Smart Mobility and
Logistics, KFUPM, KSA.
His research contributions are in control of the autonomous single-domain
and multi-domain multi-agent systems, unmanned systems UxV, biological
models of a fleet of unmanned aerial vehicles, process control and control
loop performance-monitoring, control of systems with delays, modeling,
control of stochastic systems, analysis of network stability, condition
monitoring, and condition-based maintenance. His research interests include
sensing, monitoring, and control with strong multidisciplinary research and
applications.
\endbio

\end{document}